\documentclass[10pt,twocolumn,letterpaper]{article}

\usepackage{iccv}
\usepackage{times}
\usepackage{epsfig}
\usepackage{graphicx}
\usepackage{amsmath}
\usepackage{amssymb}

\usepackage{amsthm, bbm, mathtools, commath, amsfonts}
\usepackage{booktabs}
\usepackage{multirow}
\usepackage{algorithm, algorithmic}
\usepackage{verbatim}
\usepackage[caption=false]{subfig}
\usepackage[page]{appendix}
\usepackage{arydshln} 
\usepackage{enumitem} 

\usepackage[pagebackref=true,breaklinks=true,colorlinks,bookmarks=false]{hyperref}

\iccvfinalcopy 

\ificcvfinal\fi

\DeclareMathOperator*{\argmin}{arg\,min}

\newcommand{\mrm}[1]{\mathrm{#1}}
\newcommand{\RR}{\mathbb R}
\newcommand{\mc}[1]{\mathcal{#1}}

\newcommand{\supscript}[1]{\textsuperscript{#1}}
\newcommand{\mbf}[1]{\mathbf{#1}}
\def \figpath {figures/}

\def\C{\mathbf{C}}
\def\A{\mathbf{A}}
\def\Re{\mathbb{R}}
\def\X{\mathbf{X}}

\newcommand{\myparagraph}[1]{\smallskip \noindent \textbf{#1}}

\newtheorem{proposition}{Proposition}
\newtheorem{lemma}{Lemma}

\begin{document}

\title{Doubly Stochastic Subspace Clustering}

\author{Derek Lim\\
Department of Computer Science\\
Cornell University\\
{\tt\small dl772@cornell.edu}
\and
Ren\'{e} Vidal \qquad \qquad Benjamin D. Haeffele\\
Mathematical Institute for Data Science\\
Johns Hopkins University\\
{\tt\small \{rvidal, bhaeffele\}@jhu.edu}
}

\maketitle

\begin{abstract}
Many state-of-the-art subspace clustering methods follow a two-step process by first constructing an affinity matrix between data points and then applying spectral clustering to this affinity. Most of the research into these methods focuses on the first step of generating the affinity, which often exploits the self-expressive property of linear subspaces, with little consideration typically given to the spectral clustering step that produces the final clustering. Moreover, existing methods often obtain the final affinity that is used in the spectral clustering step by applying ad-hoc or arbitrarily chosen postprocessing steps to the affinity generated by a self-expressive clustering formulation, which can have a significant impact on the overall clustering performance. In this work, we unify these two steps by learning both a self-expressive representation of the data and an affinity matrix that is well-normalized for spectral clustering. In our proposed models, we constrain the affinity matrix to be doubly stochastic, which results in a principled method for affinity matrix normalization while also exploiting known benefits of doubly stochastic normalization in spectral clustering. We develop a general framework and derive two models: one that jointly learns the self-expressive representation along with the doubly stochastic affinity, and one that sequentially solves for one then the other. Furthermore, we leverage sparsity in the problem to develop a fast active-set method for the sequential solver that enables efficient computation on large datasets. Experiments show that our method achieves state-of-the-art subspace clustering performance on many common datasets in computer vision.
\end{abstract}

\section{Introduction}

Subspace clustering seeks to cluster a set of data points that are approximately drawn
from a union of low dimensional linear (or affine) subspaces into clusters,
where each linear (or affine) subspace defines a cluster
(i.e., every point in a given cluster lies in the same subspace) \cite{vidal2011subspace}.
The most common class of subspace clustering algorithms
for clustering a set of $n$ data points proceed
in two stages: 1) Learning an affinity matrix $\A \in \Re^{n \times n}$
that defines the similarity between pairs of datapoints; 
2) Applying a graph clustering technique, such as
spectral clustering \cite{von2007tutorial,shi2000normalized},
to produce the final clustering.  
In particular, arguably the most
popular model for subspace clustering is to construct the affinity matrix
by exploiting the `self-expressive' property of linear (or affine)
subspaces, where a point within a given subspace can be represented
as a linear combination of other points within the subspace
\cite{elhamifar2009sparse, elhamifar2013sparse, Lu12, you2016oracle, you2016scalable, vidal2011subspace}.
For a dataset $\X \in \Re^{d \times n}$ of $n$, $d$-dimensional data points,
this is typically captured by an optimization problem of the form:
\begin{equation}
\label{eq:basic_form}
\min_{\C} \tfrac{1}{2} \|\X - \X\C\|_F^2 + \lambda \theta(\C),
\end{equation}
where the first term captures the self-expressive property, $\X \approx \X\C$,
and the second term, $\theta$, is some regularization term on $\C$ to
encourage that a given point is primarily represented by other points from
its own subspace and to avoid trivial solutions
such as $\C = \mbf{I}$.   Once the self-expressive representation
$\C$ has been learned from \eqref{eq:basic_form},
the final affinity $\A$ is then typically constructed
by rectifying and symmetrizing $\C$, e.g., $\A = (|\C| + |\C^\top|)/2$.  

\begin{figure*}[ht]
	\centering
	\includegraphics[width=1.8\columnwidth]{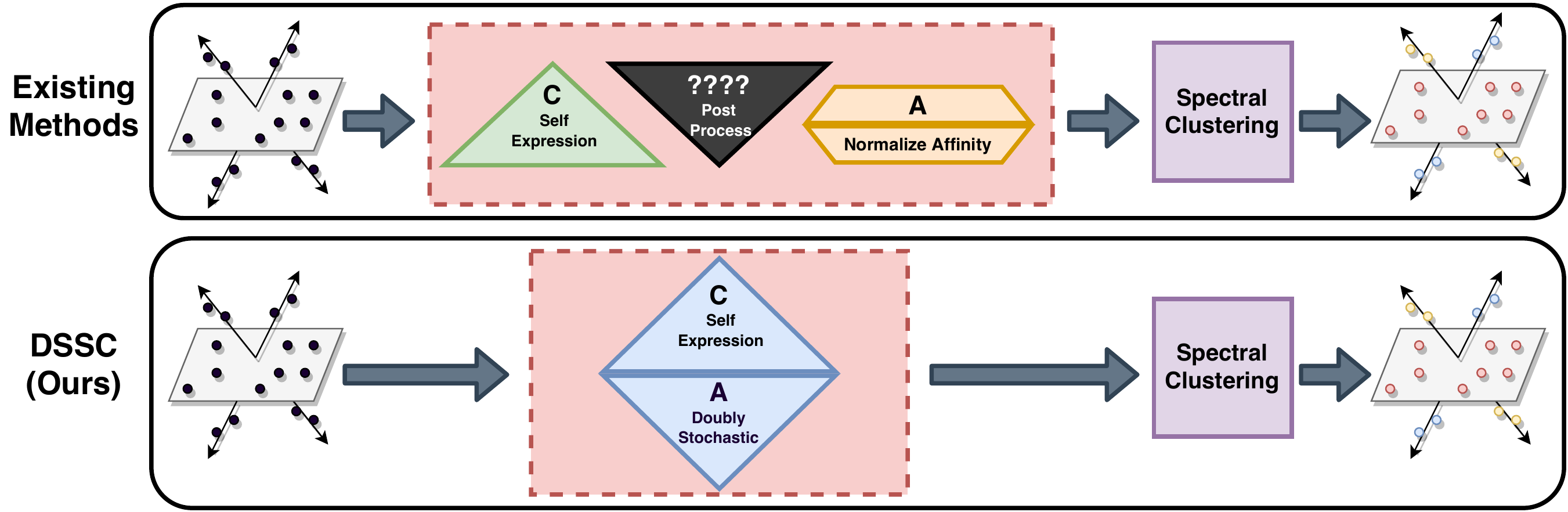}
	\caption{Diagram comparing our DSSC framework with existing methods for self-expressive affinity-based subspace clustering.
	(Top row) Most existing methods focus on computing the self-expressive matrix $\mbf{C}$,
			but they are also are reliant on choices of postprocessing and affinity normalization,
			which often take the form of ad-hoc procedures that are not well-studied.
	(Bottom row) Our DSSC models learn a doubly stochastic affinity matrix
	$\mbf{A} \in \boldsymbol{\Omega}_n$ along with the
	self-expressive matrix $\mbf{C}$. 
	The doubly stochastic affinity does not require postprocessing
or normalization to be used for spectral clustering and has numerous desirable properties
for subspace clustering.}
\label{fig:diagram}
\end{figure*}

However, as we detail in Section~\ref{sec:motivation} and illustrate in Figure~\ref{fig:diagram}, many subspace clustering methods require ad-hoc or unjustified postprocessing procedures on $\mbf{C}$ to work in practice. 
These postprocessing steps serve to normalize $\mbf{C}$ to produce an affinity $\mbf{A}$ with better spectral clustering performance, but they add numerous arbitrary
hyperparameters to the models and receive little mention in the associated papers.  Likewise, it is well-established that some form of Laplacian normalization is needed for spectral clustering to be successful \cite{von2007tutorial}, for which the practitioner again has multiple choices of normalization strategy. 

In this paper, we propose a new subspace clustering framework which explicitly connects the self-expressive step, the affinity normalization, and the spectral clustering step. We develop novel scalable methods for this framework that address issues with and empirically outperform other subspace clustering models. To motivate our proposed models, we first discuss the desired properties of subspace clustering affinities that one would like for successful spectral clustering.
Specifically, we desire affinities $\mbf{A}$ that have the following properties:
\begin{itemize}[noitemsep,topsep=3pt]
	\item[(A1)] \textit{Well-normalized for spectral clustering}. We should better leverage
		knowledge that $\mbf{A}$ will be input to spectral clustering.
		Many forms of ad-hoc postprocessing in subspace clustering
		perform poorly-justified normalization, and there is also
		a choice of normalization to be made in forming the graph Laplacian \cite{von2007tutorial}.
	\item[(C1)] \textit{Sparsity}. Much success in subspace clustering
		has been found with enforcing sparsity on $\mbf{C}$
	\cite{elhamifar2009sparse, elhamifar2013sparse, you2016scalable}.
	In particular, many sparsity-enforcing methods are designed to ensure that the nonzero entries of the affinity matrix correspond to pairs of points which belong to the same subspace.  This leads to high clustering accuracy,
	desirable computational properties, and provable theoretical guarantees.
	\item[(C2)] \textit{Connectivity}. There should be sufficiently many edges
		in the underlying graph of $\mbf{A}$ so that all points within the
		same subspace are connected
		\cite{nasihatkon2011graph, wang2016graph, you2016oracle}.
		Thus, there is a trade-off between sparsity and connectivity that subspace clustering methods
		must account for.
\end{itemize}

Properties (C1) and (C2) are well studied in the subspace clustering literature, and can be enforced on the self-expressive $\mbf{C}$, since rectifying and symmetrizing a $\mbf{C}$ with these properties, for example, maintains the properties. However, property (A1) must be enforced on $\mbf{A}$; thus, (A1) is often neglected in several aspects (see Section~\ref{sec:motivation}),
and is mostly handled by ad-hoc postprocessing methods, if handled at all. 
In working towards (A1), we first constrain $\mbf{A} \geq \mbf{0}$, since nonnegativity of the affinity is necessary for interpretability and alignment with spectral clustering theory. Beyond this, spectral clustering also benefits from having rows and columns in the affinity matrix of the same scale, 
and many forms of Laplacian normalization for spectral clustering normalize rows and/or columns to have similar scale  \cite{von2007tutorial}.  
In particular, one form of normalization that is well established in the spectral clustering literature is to constrain the rows and columns to have unit $l_1$ norm. Because $\mbf{A}$ is additionally constrained to be nonnegative, this is equivalent to requiring each row sum and column sum of $\mbf{A}$ to be 1, resulting in constraints that restrict the affinities $\mbf{A}$ to be in the convex set of doubly stochastic matrices \cite{horn2012matrix}: 
\begin{equation}
	\boldsymbol{\Omega}_n = \{ \mbf{A} \in \RR^{n \times n} \mid \mbf{A} \geq \mbf{0}, \ \mbf{A}\mbf{1} = \mbf{1}, \ \mbf{A}^\top \mbf{1} = \mbf{1} \}.
\end{equation}
Doubly stochastic matrices have been thoroughly studied for spectral clustering, and doubly stochastic normalization has been shown to significantly improve the performance of spectral clustering
\cite{zass2005unifying, zass2007doubly, wang2016structured,
nie2016constrained, harlev2020doubly, jankowski2020spectra, wang2010learning, landa2020doubly, yan2014efficient}.
Moreover, doubly stochastic matrices are invariant to the most widely-used types of Laplacian normalization \cite{von2007tutorial}, which removes the need to choose a Laplacian normalization scheme.  Further, the authors of \cite{zass2007doubly} show that various Laplacian normalization schemes can be viewed as attempting to approximate a given affinity matrix with a doubly stochastic matrix under certain distance metrics.

Beyond being well-normalized for spectral clustering (A1), the family of doubly stochastic matrices can also satisfy properties (C1) and (C2).  
Our proposed methods will give control over the sparsity-connectivity trade-off through interpretable parameters, and we in practice learn $\mbf{A}$ that are quite sparse.
Additionally, doubly stochastic matrices have a guarantee of a certain level of connectivity due to the row sum constraint, which prohibits solutions with all-zero rows (that can occur in other subspace clustering methods). The convexity of $\boldsymbol{\Omega}_n$ along with the sparsity of our learned $\mbf{A}$ allow us to develop novel scalable algorithms for doubly stochastic projection and hence scalable methods for subspace clustering with doubly stochastic affinities.

\myparagraph{Contributions.}
In this work, we develop a framework that unifies the self-expressive representation step of subspace clustering with the spectral clustering step by formulating a model that jointly solves for a self-expressive representation $\mbf{C}$ and a doubly stochastic affinity matrix $\mbf{A}$.
While our general model is non-convex, we provide a convex relaxation that is provably close to the non-convex model and that can be solved by a type of linearized ADMM \cite{ma2016alternating}. A closer analysis also allows us to formulate a sequential algorithm to quickly compute an approximate solution, in which we first efficiently learn a self-expressive matrix and then subsequently fit a doubly stochastic matrix by a regularized optimal transport problem \cite{blondel2018smooth, lorenz2019quadratically}. We leverage inherent sparsity in the problem to develop a scalable active-set method for computing in this sequential model. 
Finally, we validate our approach with experiments on various standard subspace clustering datasets, where we demonstrate that our models significantly improve on the current state-of-the-art in a variety of performance metrics.


\subsection{Related Work}

\myparagraph{Self-expressive and Affinity Learning.} 
Existing methods that have attempted to unify the self-expressive and spectral clustering steps in subspace clustering include: Structured SSC \cite{li2017structured}, which jointly learns a self-expressive representation and a binary clustering matrix, and Block Diagonal Representation \cite{lu2018subspace} and SC-LALRG \cite{yin2018subspace}, which jointly learn a self-expressive representation and a constrained affinity for spectral clustering. However, all of these methods are non-convex, require iterative optimization methods, and necessitate expensive spectral computations in each iteration. As such, their scalability is greatly limited. In contrast, our methods have convex formulations, and we can compute the true minimizers very efficiently with our developed algorithms.

\myparagraph{Doubly Stochastic Clustering.}
Doubly stochastic constraints have been used in standard spectral clustering methods where the input affinity is known, but they have not been used to directly learn an affinity matrix for a self-expressive representation, as is desired in subspace clustering.

Specifically, Zass and Shashua find strong performance by applying standard spectral clustering to the nearest doubly stochastic matrix to the input affinity in Frobenius norm \cite{zass2007doubly}. Also, there has been work on learning doubly stochastic approximations of input affinity matrices subject to rank constraints \cite{nie2016constrained, wang2016structured, yang2016low}. Further, Landa \textit{et al.} show that doubly stochastic normalization of the ubiqitous Gaussian kernel matrix by diagonal matrix scaling is robust to heteroskedastic additive noise on the data points \cite{landa2020doubly}.

To the best of our knowledge, only one subspace clustering method \cite{lee2015membership} utilizes doubly stochastic constraints, but they enforce very expensive semi-definite constraints, use ad-hoc non-convex postprocessing, and do not directly apply spectral clustering to the doubly stochastic matrix. In contrast, we develop scalable methods with principled convex formulations that do not require postprocessing before applying spectral clustering.

\myparagraph{Scalable Subspace Clustering.} 
Existing methods or making subspace clustering scalable leverage sparsity in the self-expressive representation \cite{chen2020stochastic,you2016oracle,dyer2013greedy,you2016scalable}, leverage structure in the final affinity \cite{peng2013scalable,adler2015linear}, and/or use greedy heuristics \cite{dyer2013greedy,you2016scalable,you2016divide} to efficiently compute subspace clusterings.
Our scalable method is able to exploit sparsity in both the self-expressive representation and affinity construction, is easily parallelizable (allowing a simple and fast GPU implementation), and is fully convex, with no use of greedy heuristics.

\section{Doubly Stochastic Models}

In this section, we further detail the motivation for learning doubly stochastic
affinities for subspace clustering. Then we develop two models for subspace
clustering with doubly stochastic affinities. Our J-DSSC model 
jointly learns a self-expressive matrix and doubly stochastic affinity,
while our A-DSSC model is a fast approximation that sequentially solves for a self-expressive
matrix and then a doubly stochastic affinity that approximates it.

\subsection{Benefits of Doubly Stochastic Affinities} \label{sec:motivation}

As discussed above, existing subspace clustering methods such as SSC \cite{elhamifar2013sparse}, LRR \cite{liu2012robust}, EDSC \cite{ji2014efficient}, and deep subspace clustering networks (DSC-Net) \cite{ji2017deep} require various ad-hoc postprocessing methods to achieve strong clustering performance on certain datasets.
For example, it has recently been shown that if one removes the ad-hoc postprocessing steps from DSC-Net~\cite{ji2017deep} then the performance drops considerably --- performing no better than simple baseline methods and far below state-of-the-art \cite{haeffele2021critique}. Common postprocessing steps include keeping only the top $l$ entries of each column of $\mbf{C}$, normalizing columns of $\mbf{C}$, and/or using SVD-based postprocessing, each of which introduces extra hyperparameters and degrees of freedom for the practitioner. 
Likewise, the practitioner is also required to make a choice on the particular type of Laplacian normalization to use \cite{von2007tutorial}.

To better connect the self-expressive step with the subsequent spectral clustering step and to avoid the need for ad-hoc postprocessing methods, we propose a framework for directly learning an affinity $\mbf{A}$ that already has desired properties (A1) and (C1)-(C2) for spectral clustering. Restricting $\mbf{A} \in \boldsymbol{\Omega_n}$ to be in the convex set of doubly stochastic matrices achieves
these goals while allowing for highly efficient computation.

As noted in the introduction, doubly stochastic matrices are already nonnegative, so we do not need to take absolute values of some computed matrix. Also, each row and column of a doubly stochastic matrix sums to one, so they each have the same scale in $l_1$ norm --- removing the need to postprocess by scaling rows or columns. Importantly for (A1), doubly stochastic matrices are invariant to most forms of Laplacian normalization used in spectral clustering.  For example, for a symmetric affinity matrix $\mbf{A}$ with row sums $\mbf{D} = \mrm{diag}(\mbf{A} \mbf{1})$, widely used Laplacian variants include the unnormalized Laplacian $\mbf{D} - \mbf{A}$, the normalized Laplacian $\mbf{I} - \mbf{D}^{-1/2} \mbf{A} \mbf{D}^{-1/2}$, and the random walk Laplacian $\mbf{I} - \mbf{D}^{-1} \mbf{A}$ \cite{von2007tutorial}. When $\mbf{A}$ is doubly stochastic, the matrix of row sums satisfies $\mbf{D} = \mbf{I}$, so all of the normalization variants are equivalent and give the same Laplacian~$\mbf{I} - \mbf{A}$.

In addition, many types of regularization and constraints that have been proposed for subspace clustering tend to desire sparsity (C1) and connectivity (C2), in the sense that they want $\mbf{A}_{ij}$ to be small in magnitude or zero when $\mbf{x}_i$ and $\mbf{x}_j$ belong to different subspaces and to be nonzero for sufficiently many pairs $(i,j)$ where $\mbf{x}_i$ and $\mbf{x}_j$ belong to the same subspace \cite{nasihatkon2011graph,wang2016graph}. The doubly stochastic matrices learned by our models can be tuned to achieve any desired sparsity level. Also, doubly stochastic affinities are guaranteed a certain level of connectedness, as
each row must sum to $1$. This means that there cannot be any zero rows in the learned affinity, unlike in methods that compute representations one column at a time such as SSC \cite{elhamifar2013sparse}, EnSC \cite{you2016oracle}, and SSC-OMP \cite{you2016scalable}, where it is possible that a point is never used in the self-expressive representation of other points. 

\subsection{Joint Learning: J-DSSC}\label{sec:jdssc}

In developing our model to jointly learn a self-expressive matrix and doubly stochastic affinity, we build off of the general regularized self-expression form in \eqref{eq:basic_form}.  In addition to learning a self-expressive matrix $\C$ we also wish to learn a doubly stochastic affinity matrix $\mbf{A} \in \boldsymbol{\Omega}_n$. Self-expressive formulations (roughly) model $\abs{\mbf{C}_{ij}}$ as proportional to the likelihood that $\mbf{x}_i$ and $\mbf{x}_j$ are in the same subspace, so we desire that our normalized affinity $\mbf{A}$ be close to $\abs{\mbf{C}}$ (optionally after scaling $\A$ by a constant to match the scale of $\abs{\C}$), which we incorporate via the use of a penalty function $\Theta(\cdot,\cdot)$. Thus, we have a general framework:
\begin{equation}
	\begin{aligned}
	& \min_{\mbf{C}, \mbf{A}} \frac{1}{2}\norm{\mbf{X} - \mbf{X}\mbf{C}}_F^2 + \lambda \theta(\mbf{C}) + \gamma \Theta(\abs{\mbf{C}}, \mu \mbf{A}) \\ 
	& \mrm{s.t.} \ \mbf{A} \in \boldsymbol{\Omega}_n, \ \mrm{diag}(\mbf{C}) = \mbf{0}
\end{aligned}
\end{equation}
The zero-diagonal constraint on $\mbf{C}$ is enforced to prevent each point from using itself in its self-representation, as this is not informative for clustering.
While our framework is general, and can be used with e.g. low-rank penalties, we choose a simple mix of $l_2$ and $l_1$ regularization that is effective and admits fast algorithms.
For $\Theta$, we use an $l_2$ distance penalty, as $l_2$ projection of an affinity matrix onto the doubly
stochastic matrices often improves spectral clustering performance \cite{zass2007doubly} and results in desirable sparsity properties \cite{rontsis2020optimal,blondel2018smooth}. As a result, our model takes the form of an optimization problem over $\mbf{C}$ and $\mbf{A}$:
\begin{equation}\label{eq:absC}
\begin{aligned}
	\min_{\mbf{C}, \mbf{A}} &~~ \frac{1}{2}\norm{\mbf{X} - \mbf{X}\mbf{C}}_F^2 
	 + \frac{\eta_1}{2}\norm{\abs{\mbf{C}} - \eta_2 \mbf{A}}_F^2 + \eta_3 \norm{\mbf{C}}_1\\
	\mrm{s.t.} &~~ \mbf{A} \in \boldsymbol\Omega_n, \ \mrm{diag}(\mbf{C}) = \mbf{0},
\end{aligned}
\end{equation}
where $\eta_1, \eta_2 > 0, \eta_3 \geq 0$ are hyperparameters.

The objective \eqref{eq:absC} is not convex because of the $\abs{\mbf{C}}$ term in the $\mbf{C}$-to-$\mbf{A}$-difference loss. To alleviate this issue, we relax
the problem by separating the self-expressive matrix $\C$ into two nonnegative matrices $\mbf{C}_p, \mbf{C}_q \geq \mbf{0}$, so that ${\mbf{C}_p - \mbf{C}_q}$ approximately takes the role of $\mbf{C}$ and ${\mbf{C}_p + \mbf{C}_q}$ approximately takes the role of $\abs{\mbf{C}}$ (with the approximation being exact if the nonzero support of $\mbf{C}_p$ and $\mbf{C}_q$ do not overlap). Thus, our convex model is given by:
\begin{equation}\label{eq:model}
\begin{aligned}
	\min_{\mbf{C}_p, \mbf{C}_q, \mbf{A}} & ~ \frac{1}{2} \norm{\mbf{X}-\mbf{X}[\mbf{C}_p-\mbf{C}_q]}_F^2 + \\
	& ~\frac{\eta_1}{2}\norm{[\mbf{C}_p+\mbf{C}_q] - \eta_2 \mbf{A}}_F^2 + \eta_3 \norm{\mbf{C}_p+\mbf{C}_q}_1 \!\\
	 \mrm{s.t.} ~ & ~ \mbf{A} \in \boldsymbol\Omega_n, \ \mbf{C}_p, \mbf{C}_q \in \RR^{n \times n}_{\geq 0, \mrm{diag}=0}
\end{aligned}
\end{equation}
where $\RR^{n \times n}_{\geq 0, \mrm{diag}=0}$ is the set of $n \times n$ real matrices with nonnegative entries and zero diagonal\footnote{ Note that the $l_1$ norm $\norm{\mbf{C}_p + \mbf{C}_q}_1$ can be replaced by the
sum of the entries of $\mbf{C}_p + \mbf{C}_q$ because of the nonnegativity constraints.}.
This problem is now convex and can be efficiently solved to global optimality, as we discuss in Section \ref{sec:joint_opt}. We refer to this model as Joint Doubly Stochastic Subspace Clustering (J-DSSC).

If the optimal $\mbf{C}_p$ and $\mbf{C}_q$ have disjoint support (where the support of a matrix is the set of indices $(i,j)$ where the matrix is nonzero), then by solving \eqref{eq:model} we also obtain a solution for \eqref{eq:absC}, since we can take $\mbf{C} = \mbf{C}_p - \mbf{C}_q$, in which case $\abs{\mbf{C}} = \abs{\mbf{C}_p - \mbf{C}_q} = \mbf{C}_p + \mbf{C}_q$. However, this final equality does not hold when the optimal $\mbf{C}_p$ and $\mbf{C}_q$ have overlapping nonzero support. The following proposition shows that our relaxation \eqref{eq:model} is equivalent to the original problem \eqref{eq:absC} for many parameter settings of $(\eta_1,\eta_2,\eta_3)$.
In cases where the supports of $\mbf{C}_p$ and $\mbf{C}_q$ overlap, we can also bound the magnitude of the overlapping entries (hence guaranteeing a close approximation of the solution to \eqref{eq:absC} from solutions to \eqref{eq:model}). A proof is given in the appendix.
\begin{proposition}\label{prop:supports}
	Consider  \eqref{eq:model} with parameters $\eta_1 > 0$ and $\eta_2, \eta_3 \geq 0$. Let $(\mbf{C}_p^*, \mbf{C}_q^*, \mbf{A}^*)$ be an optimal solution.
	\begin{itemize}
		\item[1)] If $\eta_1  \eta_2 \leq \eta_3$, then the solution of the relaxation \eqref{eq:model} is the same as that of \eqref{eq:absC}, when taking $\mbf{C} = \mbf{C}_p^* - \mbf{C}_q^*$. In particular, the supports of $\mbf{C}_p^*$ and $\mbf{C}_q^*$ are disjoint.
		\item[2)] If $\eta_1  \eta_2 > \eta_3$, then the supports of $\mbf{C}_p^*$ and $\mbf{C}_q^*$ may overlap. At any index $(i,j)$ for which $(\mbf{C}_p^*)_{ij} > 0$ and $(\mbf{C}_q^*)_{ij} > 0$, it holds that
			\begin{equation}\label{eq:overlap_support}
			\max \left((\mbf{C}_p^*)_{ij}, \; (\mbf{C}_q^*)_{ij} \right) < \frac{\eta_1  \eta_2 - \eta_3}{\eta_1}.
		\end{equation}
	\end{itemize}
\end{proposition}

\subsection{Sequential Approximation: A-DSSC}

To see an alternative interpretation
of the model in \eqref{eq:absC}, note the following by expanding the second term of \eqref{eq:absC}:
\begin{equation}\label{eq:expand}
\begin{aligned}
	\eqref{eq:absC} = & \min_{\mbf{C}, \mbf{A}} \; \frac{1}{2}\norm{\mbf{X} - \mbf{X}\mbf{C}}_F^2 + \frac{\eta_1}{2}\norm{\mbf{C}}_F^2 + \eta_3 \norm{\mbf{C}}_1\\ 
	& \quad - \langle \eta_1 \abs{\mbf{C}}, \eta_2 \mbf{A} \rangle + \frac{\eta_1\eta_2^2}{2}\norm{\mbf{A}}_F^2 \\
	& \mrm{s.t.} \ \mbf{A} \in \boldsymbol\Omega_n, \ \mrm{diag}(\mbf{C}) = \mbf{0}
\end{aligned}
\end{equation}
From this form, one can see that for an uninformative initialization of the affinity as $\mbf{A}=\mbf{I}$, the minimization w.r.t. $\mbf{C}$ takes the form of Elastic Net Subspace Clustering \cite{you2016oracle}, 
\begin{equation}\label{eq:minC}
	\begin{aligned}
	& \min_{\mbf{C}} \frac{1}{2}\norm{\mbf{X} - \mbf{X}\mbf{C}}_F^2 + \frac{\eta_1}{2}\norm{\mbf{C}}_F^2 + \eta_3 \norm{\mbf{C}}_1\\
	& \ \mrm{s.t.} \  \mrm{diag}(\mbf{C}) = \mbf{0}
\end{aligned}
\end{equation}
Likewise, for a fixed $\mbf{C}$, one can observe that the above problem w.r.t. $\mbf{A}$ is a special case of a quadratically regularized optimal transport problem
\cite{blondel2018smooth,lorenz2019quadratically}:
\begin{equation}\label{eq:jminA}
	\min_{\mbf{A}} \  \langle - \abs{\mbf{C}}, \mbf{A} \rangle + \frac{\eta_2}{2}\norm{\mbf{A}}_F^2 \ \mrm{s.t.} \ \mbf{A} \in \boldsymbol\Omega_n.
\end{equation}
Thus, $\eta_2$ controls the regularization on $\mbf{A}$, with lower $\eta_2$ encouraging sparser $\mbf{A}$ and higher $\eta_2$ encouraging denser and more uniform $\mbf{A}$.  In particular, as $\eta_2 \to 0$, \eqref{eq:jminA} approaches a linear assignment problem, which has permutation matrix solutions (i.e., maximally sparse solutions) \cite{peyre2019computational, horn2012matrix}.  In contrast, as $\eta_2 \to \infty$, the optimal solution is densely connected and approaches the uniform matrix $\frac{1}{n}\mbf{1}\mbf{1}^\top$.

Hence, we consider an alternating minimization process to obtain approximate solutions for $\mbf{C}$ and $\mbf{A}$, where we first initialize $\A = \mbf{I}$ and solve \eqref{eq:minC} for $\C$.  Then, holding $\C$ fixed we solve \eqref{eq:jminA} for $\A$.
Taking the solution to this problem as the final affinity $\mbf{A}$, we obtain our one-step approximation to J-DSSC,
which we refer to as Approximate Doubly Stochastic Subspace Clustering (A-DSSC)\footnote{Note that additional alternating minimization steps can also be used as \eqref{eq:expand} is convex w.r.t. $\C$ if $\A$ is held fixed (for any feasible $\A$) and vice-versa.}. For this model, we develop a fast algorithm in Section~\ref{sec:optimization}, show state-of-the-art clustering performance in Section~\ref{sec:experiments}, and show empirically in the appendix that A-DSSC well approximates the optimization problem of J-DSSC.

Besides being an approximation to the joint model \eqref{eq:model}, A-DSSC also has an interpretation as a postprocessing method for certain subspace clustering methods that can be expressed as in \eqref{eq:minC}, such as SSC, EnSC, and LSR \cite{elhamifar2013sparse, you2016oracle, Lu12}. Instead of arbitrarily making choices about how to postprocess $\mbf{A}$ and form the normalized Laplacian, we instead take $\mbf{A}$ to be doubly stochastic; as discussed above, this provides a principled means of generating an affinity $\A$ that is suitably normalized for spectral clustering 
from a well-motivated convex optimization problem \eqref{eq:jminA}.

\section{Optimization and Implementation}\label{sec:optimization}

Here we outline the optimization procedures to compute J-DSSC and A-DSSC. The basic algorithm for subspace clustering with our models is in Algorithm~\ref{alg:dssc_outline}. Further algorithmic and implementation details are in the appendix.

\subsection{Joint Solution for J-DSSC}\label{sec:joint_opt}

In order to efficiently solve \eqref{eq:model}, we develop an algorithm in the style of linearized ADMM \cite{ma2016alternating}. We reparameterize the problem by introducing additional variables, and iteratively take minimization steps over an augmented Lagrangian as well as dual ascent steps on dual variables. As this procedure is rather standard, we delegate the full description of it to the appendix.

\begin{algorithm}[t]
	\caption{Doubly Stochastic Subspace Clustering.}
	\label{alg:dssc_outline}
	\begin{algorithmic}
		\renewcommand{\algorithmicrequire}{\textbf{Input:}}
		\REQUIRE Data matrix $\mbf{X}$, parameters $\eta_1, \eta_2, \eta_3$
		\STATE Compute $\mbf{A}$ by J-DSSC \eqref{eq:model} or A-DSSC \eqref{eq:minC}, \eqref{eq:jminA}.
		\STATE Apply spectral clustering on Laplacian $\mbf{I} - \frac{1}{2}(\mbf{A}+\mbf{A}^\top)$.
		\ENSURE Clustering result.
	\end{algorithmic}
\end{algorithm}

\begin{algorithm}[t]
	\caption{Scalable A-DSSC Active-set method.}
	\label{alg:adssc_active}
	\begin{algorithmic}
		\renewcommand{\algorithmicrequire}{\textbf{Input:}}
		\REQUIRE Data matrix $\mbf{X}$, parameters $\eta_1, \eta_2, \eta_3$
		\STATE Compute $\mbf{C}$ in \eqref{eq:minC} by an EnSC or LSR solver.
		\STATE Initialize $\mbf{S}$ (see appendix).
		\STATE Initialize $\boldsymbol{\alpha}_{\circ}$ and $\boldsymbol{\beta}_{\circ} \in \RR^n$.
		\STATE Compute $\mbf{A}_\circ = \frac{1}{\eta_2}\left[\abs{\mbf{C}} - \boldsymbol{\alpha}_{\circ}\mbf{1}^\top - \mbf{1} \boldsymbol{\beta}_{\circ}^\top \right]_+$.
		\WHILE {$\mbf{A}_\circ$ is not doubly stochastic}
		\STATE Compute $\boldsymbol{\alpha}_{\circ}$ and $\boldsymbol{\beta}_{\circ}$ in \eqref{eq:support_dual} by L-BFGS.
		\STATE Compute $\mbf{A}_\circ = ${\small $\frac{1}{\eta_2}\left[\abs{\mbf{C}} - \boldsymbol{\alpha}_{\circ}\mbf{1}^\top - \mbf{1} \boldsymbol{\beta}_{\circ}^\top \right]_+$}.
		\STATE $\mrm{supp}(\mbf{S}) \gets \mrm{supp}(\mbf{A}_\circ) \cup \mrm{supp}(\mbf{S})$.
		\ENDWHILE
		\ENSURE Optimal $\mbf{A} = \mbf{A}_\circ$.
	\end{algorithmic}
\end{algorithm}

\subsection{Approximation Solution for A-DSSC}\label{sec:adssc_compute}

\myparagraph{Solving for C.}
As previously suggested, the approximate model A-DSSC can be more efficiently solved than J-DSSC. First, the minimization over $\mbf{C}$ in \eqref{eq:minC} can be solved by certain scalable subspace clustering algorithms. In general, it is equivalent to EnSC, for which scalable algorithms have been developed \cite{you2016oracle}.
Likewise, in the special case of $\eta_1 > 0, \eta_3 = 0$, it is equivalent to LSR \cite{Lu12}, which can be computed by a single $n \times n$ linear system solve.

\myparagraph{Solving for A.}
Prior work has focused on computing the doubly stochastic projection $\eqref{eq:jminA}$ through the primal \cite{zass2007doubly, rontsis2020optimal}, but we will instead solve it through the dual. The dual gives an unconstrained optimization problem over two vectors $\boldsymbol{\alpha}, \boldsymbol{\beta} \in \RR^n$, which is easier to solve as it eliminates the coupled constraints in the primal. We give computational evidence in the appendix that the dual allows for signficantly faster computation, with over an order of magnitude improvement in certain settings. Moreover, the dual allows us to develop a highly scalable active-set method in Section~\ref{sec:active_set}, which also admits a simple GPU implementation --- thus allowing for even further speed-up over existing methods. Now, the dual takes the form:
\begin{equation}\label{eq:dual}
	\max_{\boldsymbol{\alpha}, \boldsymbol{\beta}} \  -\mbf{1}^\top(\boldsymbol{\alpha} + \boldsymbol{\beta}) - \frac{1}{2\eta_2}\norm{\left[\abs{\mbf{C}} - \boldsymbol{\alpha} \mbf{1}^\top - \mbf{1} \boldsymbol{\beta}^\top \right]_+ }_F^2
\end{equation}
in which $[ \cdot ]_+$ denotes half-wave rectification, meaning that ${[\mbf{x}]_+ = \max \{\mbf{0},\mbf{x}\}}$, applied entry-wise. The optimal matrix $\mbf{A}$ is then recovered as:
\begin{equation}\label{eq:recoverA}
	\mbf{A} = \frac{1}{\eta_2}\left[\abs{\mbf{C}} - \boldsymbol{\alpha}\mbf{1}^\top - \mbf{1} \boldsymbol{\beta}^\top \right]_+.
\end{equation}
We use L-BFGS \cite{liu1989limited} to solve for $\boldsymbol{\alpha}$ and $\boldsymbol{\beta}$ in \eqref{eq:dual}.

\subsection{Scalable Active-Set Method for A-DSSC}\label{sec:active_set}

While A-DSSC can be computed more efficiently than J-DSSC, the doubly stochastic projection step still requires $\mc O(n^2)$ operations to evaluate the objective \eqref{eq:dual} and its subgradient. This limits its direct applicability to larger datasets with high $n$. We take advantage of the sparsity of the optimal $\mbf{A}$ in the parameter regimes that we choose, and develop a significantly more efficient active-set method for computing the doubly stochastic projection in $\eqref{eq:jminA}$.

Let $\mbf{S} \in \{0,1\}^{n \times n}$ be a binary support matrix and let $\mbf{S}^c$ be its complement. The basic problem that we consider is
{
\begin{equation}\label{eq:support_primal}
	\begin{aligned}
   	 \min_{\mbf{A}} \; \langle -|\mbf{C}|, \mbf{A} \rangle + \frac{\eta_2}{2}\norm{\mbf{A}}_F^2 \\
	 \mrm{s.t.} \ \mbf{A} \in \boldsymbol{\Omega}_n, \; \mbf{A} \odot \mbf{S}^c = \mbf{0}.
	\end{aligned}
\end{equation}
}
Again, we use the dual to develop an efficient solver:
\begin{equation}\label{eq:support_dual}
	\begin{aligned}
	& \max_{\boldsymbol{\alpha}_\circ, \boldsymbol{\beta}_\circ \in \RR^n} \; -(\boldsymbol{\alpha}_\circ+\boldsymbol{\beta}_\circ)^\top \mbf{1} \\
	& \qquad - \frac{1}{2\eta_2}\norm{\left[|\mbf{C}| - \boldsymbol{\alpha}_\circ \mbf{1}^\top - \mbf{1}\boldsymbol{\beta}_\circ^\top \right]_+ \odot \mbf{S}}_F^2.
	\end{aligned}
\end{equation}
Here, the dual objective and its subgradient can be evaluated in $\mc O(|\mbf{S}|)$ time, where $|\mbf{S}|$ is the number of nonzeros in $\mbf{S}$. In practice, we tend to only need $|\mbf{S}|$ that is substantially smaller than $n^2$, so this results in major efficiency increases. Also, we need only compute the elements in $\mbf{C}$ that are in the support $\mbf{S}$; these elements can be precomputed and stored at a low memory cost for faster L-BFGS iterations. In particular, this allows using LSR (i.e. setting $\eta_3 = 0$) in large datasets; even though the LSR $\mbf{C}$ is dense and takes $\mc O(n^2)$ memory, we need only compute $\mbf{C}\odot \mbf{S}$, which can be done efficiently with closed form matrix computations and takes $\mc O(|\mbf{S}|)$ memory (see appendix).
As stated above, the objective and subgradient computations can be easily implemented on GPU --- allowing even further speed-ups.

For a dual solution $(\boldsymbol{\alpha}_{\circ}, \boldsymbol{\beta}_{\circ})$ of \eqref{eq:support_dual}, the primal solution is ${\frac{1}{\eta_2}[|\mbf{C}|- \boldsymbol{\alpha}_{\circ} \mbf 1^\top - \mbf 1 \boldsymbol{\beta}_{\circ}^\top]_+ \odot \mbf{S}}$. We show in the appendix that if the affinity with unrestricted support $\mbf{A}_{\circ} = \frac{1}{\eta_2}[|\mbf{C}|- \boldsymbol{\alpha}_{\circ} \mbf 1^\top - \mbf 1 \boldsymbol{\beta}_{\circ}^\top]_+ $ is doubly stochastic, then $\mbf{A}_{\circ}$ is the primal optimal solution to \eqref{eq:jminA}. Thus, if we iteratively update an initial support $\mbf{S}$, then the row and column sums of $\mbf{A}_\circ$ provide stopping criteria that indicate optimality.
On the other hand, if $\mbf{A}_{\circ}$ is not doubly stochastic, then the support of the true minimizer $\mbf{A}$ of the unrestricted problem \eqref{eq:jminA} is not contained in $\mrm{supp}(\mbf{S})$, so $\mrm{supp}(\mbf{S})$ is too small. In this case, we update the support by $\mrm{supp}(\mbf{S}) \gets \mrm{supp}(\mbf{A}_{\circ}) \cup \mrm{supp}(\mbf{S})$. 

Thus, we have an algorithm that initializes a support and then iteratively solves problems of restricted support \eqref{eq:support_dual} until the unrestricted affinity $\mbf{A}_{\circ}$ is doubly stochastic. This is formalized in Algorithm~\ref{alg:adssc_active}. The appendix discusses choices of initial support. We can prove the following correctness result:
\newcommand{\proptwotext}{Algorithm~\ref{alg:adssc_active} computes an optimal $\mbf{A}$ for \eqref{eq:jminA} in a finite number of steps.}
\begin{proposition}\label{prop:converge}
	\proptwotext
\end{proposition}
A proof is given in the appendix. In practice, we find that only one or two updates of the support are necessary to find the optimal $\mbf{A}$, and the algorithm runs very quickly, solving problem \eqref{eq:jminA} several orders of magnitude faster than previously proposed algorithms \cite{zass2007doubly, rontsis2020optimal}.

\section{Experiments}\label{sec:experiments}

In this section, we empirically study our J-DSSC and A-DSSC models. We show that they achieve state-of-the-art subspace clustering performance on a variety of real datasets with multiple performance metrics. In the appendix, we provide more experiments that demonstrate additional strengths and interesting properties of our models.

\subsection{Experimental Setup}\label{sec:setup}

\myparagraph{Algorithms.} We compare against several state-of-the-art subspace clustering algorithms, which are all affinity-based: SSC \cite{elhamifar2009sparse, elhamifar2013sparse}, EnSC \cite{you2016oracle}, LSR \cite{Lu12}, LRSC \cite{vidal2014low}, TSC \cite{heckel2015robust}, SSC-OMP \cite{you2016scalable}, and S\supscript{3}COMP \cite{chen2020stochastic}. For these methods, we run the experiments on our own common framework, and we note that the results we obtain are similar to those that have been reported previously in the literature on the same datasets for these methods. We report partial results from the paper of S\supscript{3}COMP \cite{chen2020stochastic}, which is included as it is a recent well-performing model. Besides TSC, which uses an inner product similarity, these methods all compute an affinity based on some self-expressive loss. 

Although methods for clustering data in a union of subspaces are not directly comparable to subspace clustering neural networks \cite{haeffele2021critique}, which cluster data supported in a union of non-linear manifolds, we still include some comparisons.  In particular, we run experiments with DSC-Net \cite{ji2017deep} as a representative neural network based method and include further comparisons (that are qualitatively similar) to other neural networks \cite{zhou2018deep, zhang2018scalable, zhang2019self, abavisani2020deep} in the appendix. However, we note that recent work \cite{haeffele2021critique} has shown that DSC-Net (and many related works for network-based subspace clustering) are often fundamentally ill-posed, calling into question the validity of the results from these methods.

\myparagraph{Metrics.} 
As is standard in evaluations of clustering, we use clustering accuracy (ACC) \cite{you2016scalable} and normalized mutual information (NMI) \cite{kvalseth1987entropy} metrics, where we take the denominator in NMI to be the arithmetic average of the entropies. Also, we consider a subspace-preserving error (SPE), 
which is given by $\frac{1}{n}\sum_{i=1}^n \sum_{j \not\in \mc S_{y_i}} \abs{\mbf{A}_{ij}}/ \norm{\mbf{A}_i}_1$, where $\mc S_{y_i}$ denotes the (subspace) cluster that point $\mbf{x}_i$ belongs to, and $\mbf{A}_i$ is the ith column of $\mbf{A}$. This measures the proportion of mass in the affinity that is erroneously given to points in different (subspace) clusters, and is often used in evaluation of subspace clustering algorithms \cite{you2016scalable,soltanolkotabi2012geometric}. We report the sparsity of learned affinities by the average number of nonzeros per column (NNZ). Although sparse affinities are generally preferred, the sparsest affinity is not necessarily the best.

\myparagraph{Datasets.} We test subspace clustering performance on the Extended Yale-B dataset \cite{georghiades2001few}, Columbia Object Image Library (COIL-40 and COIL-100) \cite{nenecolumbia}, UMIST face dataset \cite{graham1998characterising}, ORL face dataset \cite{samaria1994parameterisation}, MNIST \cite{lecun1998gradient}, and EMNIST-Letters \cite{cohen2017emnist}. 
We run experiments with the raw pixel data from the images of certain datasets; for COIL, UMIST, MNIST, and EMNIST we run separate experiments with features obtained from a scattering convolution network \cite{bruna2013invariant} reduced to dimension 500 by PCA \cite{wold1987principal}. This has been used in previous subspace clustering works \cite{you2016oracle, chen2020stochastic} to better embed the raw pixel data in a union of linear subspaces.

We do not evaluate DSC-Net \cite{ji2017deep} on scattered features as the network architecture is only compatible with image data. For MNIST ($n$=70,000) and EMNIST ($n$=145{,}600), we can only run scalable methods, as neural networks and other methods that form dense $n$-by-$n$ matrices have unmanageable memory and runtime costs.

\myparagraph{Clustering setup.} For each baseline method we learn the $\mbf{C}$ from that method,  form the affinity as ${\mbf{\hat A} = (\abs{\mbf{C}} + \abs{\mbf{C}}^\top)/2}$, take the normalized Laplacian ${\mbf{L} = \mbf{I} - \mbf{D}^{-1/2} \mbf{\hat A} \mbf{D}^{-1/2}}$ (where ${\mbf{D} = \mrm{diag}(\mbf{\hat A 1}})$), and then apply spectral clustering with the specified number of clusters $k$ \cite{von2007tutorial}. For our DSSC methods, we form $\mbf{\hat A} = {(\mbf{A} + \mbf{A}^\top)/2}$ and use $\mbf{L} = \mbf{I} - \hat \A$ directly for spectral clustering (recall our affinity is invariant to Laplacian normalization).  
For the spectral clustering we take the $k$ eigenvectors corresponding to the $k$ smallest eigenvalues of $\mbf{L}$ as an embedding of the data (we take $k+1$ eigenvectors for MNIST as in \cite{chen2020stochastic}, as well as for EMNIST for similar reasons).
After normalizing these embeddings to have unit $l_2$ norm, we obtain clusterings by random initializations of $k$-means clustering, compute the accuracy and NMI for each clustering, then report the average accuracy and NMI as the final result. We emphasize that for each method we obtain a nonnegative, symmetric affinity by a shared postprocessing of $\C$ and/or $\mbf{A}$ instead of using different ad-hoc postprocessing methods that may confound the comparisons. Since DSC-Net is highly reliant on its ad-hoc postprocessing \cite{haeffele2021critique}, we also report results with its postprocessing strategy applied (shown as DSC-Net-PP).

\myparagraph{Parameter choices.} As is often done in subspace clustering evaluation,
we choose hyperparameters for each method by searching over some set of parameters and
reporting the results that give the highest clustering accuracy. 
For DSC-Net, we use the suggested parameters in \cite{ji2017deep} where applicable,
and search over hyperparameters for novel datasets.
We give details on model hyperparameters in the appendix.

\setlength\dashlinegap{3pt}
\setlength\arrayrulewidth{.5pt}
\setlength\dashlinedash{4pt}

\begin{table*}[t]
	\centering
	\caption{Subspace clustering results on different datasets. `---' indicates that we do not have results, and `N/A' indicates that the method cannot be run for the given dataset. Best results with respect to ACC, NMI, and SPE are \bf{bolded}.}
	{\footnotesize
	\begin{tabular}{cccccccccccccc}
		\toprule
				& & \multicolumn{7}{c}{Shallow Affinity-Based} & \multicolumn{2}{c}{Neural}  & \multicolumn{2}{c}{Doubly Stochastic} \\
				\cmidrule(lr){3-9} \cmidrule(lr){10-11} \cmidrule(lr){12-13}
		Dataset & Metric & SSC  & EnSC & LSR & LRSC & TSC & SSC-OMP & S\supscript{3}COMP & DSC-Net & DSC-Net-PP  & J-DSSC &  A-DSSC \\
	\midrule
		\multirow{4}{*}{Yale-B}	& ACC & .654 & .652 & .659 & .662 & .514 & .780 & .874 & .691  & \bf{.971}  &  .924  & .917 \\
								& NMI & .734 & .734 & .743 & .739 & .629 & .844 & --- & .746  & \bf{.961} &  .952  & .947 \\
								& SPE & .217 & .218 & .869 & .875 & .217 & .179 & .203 & .881  & \bf .038 &  .080  & .080 \\
							& NNZ & 22.9 & 23.7 & 2413 & 2414 & 4.3 & 9.1 & --- & 2414  & 22.1  & 14.5  & 14.4 \\
	\midrule
		\multirow{4}{*}{COIL-40} & ACC & .799 & .801 & .577 & .567 & .813 & .411 & --- & .543  & .751 & .899  &  \bf{.922} \\
								 & NMI & .940 & .930 & .761 & .736 & .916 & .605 & --- & .743  & .887 & .963  & \bf{.967} \\
								  & SPE & .013 & .017 & .926 & .877 & .057 & .025 & --- & .873 & .105 & .012  & \bf .008 \\
								  & NNZ & 3.1 & 6.2 & 2879 & 2880 &  4.5 & 2.9 & --- & 2880 & 13.1 & 8.9  & 5.2 \\[2pt]
								  \hdashline
								  & ACC & .996 & .993 & .734 & .753 & .941 & .489 & --- & N/A & N/A &  \bf{1.00} & \bf{1.00}  \\
COIL-40		& NMI & .998 & .995 & .868 & .871 & .981 & .711 & --- & N/A & N/A & \bf{1.00} & \bf{1.00} \\
(Scattered)	& SPE & .0002 & \bf 0.00 & .787 & .816 & .004 & .056 & --- & N/A & N/A &  .00006 & .00003 \\
			& NNZ & 2.2 & 4.0 & 2879 & 2880  & 4.3 & 9.1 & --- & N/A & N/A &  13.8 & 8.9 \\
	\midrule
		\multirow{4}{*}{COIL-100} & ACC & .704 & .680 & .492 & .476 & .723 & .313 & .789 & .493 & .635 & .796 & \bf .824 \\
								  & NMI & .919 & .901 & .753 & .733 & .904 & .588 & --- & .752  & .875 & .943  & \bf .946 \\
								  & SPE & .052 & .044 & .945 & .946 & .057 & .052 & \bf .032 & .958 & .384 & .049  & .037 \\
								  & NNZ & 3.1 & 6.8 & 7199 & 7200 & 3.6 & 3.0 & --- & 7200 & 44.7 & 9.9  & 5.8 \\[3pt]
								  \hdashline
								  & ACC & .954 & .967 & .642 & .654 & .915 & .397 & --- & N/A & N/A  & .961 & \bf{.984}  \\
		COIL-100	  & NMI & .991 & .990 & .846 & .850 & .975 & .671 & --- & N/A & N/A  & .992 & \bf{.997}  \\
		(Scattered)	  & SPE & .002 & .004 & .891 & .905 & .009 & .055 & --- & N/A & N/A  & .004 & \bf .001 \\
					  & NNZ & 2.3  & 4.2 & 7199 & 7200 & 4.3  & 9.1 & --- & N/A & N/A &  10.4 & 6.6 \\
	\midrule
		\multirow{4}{*}{UMIST} & ACC & .537 & .562 & .462 & .494 & .661 & .509 & --- & .456 & .708 & \bf .732  & .725 \\
							   & NMI & .718 & .751 & .645 & .662 & .829 & .680 & --- & .611 & .848 & \bf .858  & .851 \\
							   & SPE & .079 & .085 & .811 & .824 & .035 & .091 & --- & .834 & .393 & .036  & \bf .034 \\
							   & NNZ & 4.9 & 6.6 & 574 & 575 & 3.7 & 2.9 & --- & 575 & 23.2 & 5.1 & 4.8 \\[3pt]
							   \hdashline
					   & ACC & .704 & .806 & .524 & .531 & .714 & .401 & --- & N/A & N/A  & .873 & \bf{.888} \\
		UMIST		   & NMI & .834 & .903 & .701 & .711 & .855 & .510 & --- & N/A & N/A &  \bf{.939} & .935  \\
		(Scattered)	   & SPE & .038 & .029 & .758 & .804 & .030 & .294 & --- & N/A & N/A &  .020 & \bf .018 \\
					   & NNZ & 2.7 & 4.5 & 574 & 575 & 3.8 & 14.5 & --- & N/A & N/A &  7.2 & 5.4 \\
					\midrule
		\multirow{4}{*}{ORL} & ACC & .774 & .774 & .709 & .679 & .783 & .664 & --- & .758 & \bf{.845} & .785 & .790  \\
							 & NMI & .903 & .903 & .856 & .834 & .896 & .832 & --- & .878 & \bf{.915} & .906 & .910 \\
							 & SPE & .240 & .268 & .884 & .874 & .272 & .305 & --- & .885 & .421 & .176 & \bf .159 \\
							 & NNZ & 12.3 & 14.4 & 399 & 400 & 9.3 & 5.0 & --- & 400 & 20.7 & 10.8 & 9.8  \\
		\bottomrule
	\end{tabular}
}
\vspace{-1pt}
	\label{tab:results}
\end{table*}

\begin{table*}[t]
	\centering
	\caption{Large-scale subspace clustering results. `---' indicates that we do not have results for the method. Best results are \bf{bolded}.}
	{\footnotesize
	\begin{tabular}{ccccccccc}
		\toprule
				& & \multicolumn{5}{c}{Shallow Affinity-Based} &  \multicolumn{1}{c}{Doubly Stochastic} \\
				\cmidrule(lr){3-7}  \cmidrule(lr){8-8}
		Dataset & Metric & SSC  & EnSC & TSC & SSC-OMP & S\supscript{3}COMP  &  A-DSSC \\
	\midrule
				& ACC & .963 & .963 & .980 & .574 & .963 & \bf{.990} \\
		MNIST		& NMI & .915 & .915 & .946 & .624 & --- & \bf{.971}  \\
		(Scattered)	& SPE & .098 & .098 & .028 & .301 & .301 & \bf .017 \\
					& NNZ & 39.3 & 40.2 & 18.1 & 25.4 & --- & 14.7 \\
	\midrule
					& ACC & .638 & .644 & .698 & .304 & --- & \bf .744 \\
		EMNIST		& NMI & .745 & .746 & .785 & .385 & --- & \bf .832 \\
		(Scattered)	& SPE & .173 & .174 & .106 & .216 & --- & \bf .100 \\
					& NNZ & 22.3 & 24.7 & 6.3 & 3.7 & --- & 17.5 \\

	\bottomrule
	\end{tabular}
}
\vspace{-5mm}
	\label{tab:scalableresults}
\end{table*}

\subsection{Results}

Experimental results are reported in Tables \ref{tab:results} and \ref{tab:scalableresults}. On each dataset, one of our DSSC methods achieves the highest accuracy and NMI among all non-neural-network methods. In fact, on scattered data, DSSC achieves higher accuracy and NMI than the neural method DSC-Net, even when allowing DSC-Net to use its ad-hoc postprocessing. DSC-Net with postprocessing does achieve better performance on Yale-B and ORL than our methods, though without postprocessing our models outperform it; further, note that our method is considerably simpler than DSC-Net which requires training a neural network and hence requires choices of width, activations, stochastic optimizer, depth, and so on. Moreover, our models substantially outperform DSC-Net on both the raw pixel data and scattered embeddings of UMIST, COIL-40, and COIL-100 --- achieving perfect clusterings of  COIL-40 and near-perfect clusterings of COIL-100 --- establishing new states-of-the-art to the best of our knowledge. Likewise, among the scalable methods our models achieve very strong performance on (E)MNIST.

We also see that A-DSSC, which can be viewed as running SSC, EnSC or LSR followed by a principled doubly stochastic postprocessing, outperforms SSC, EnSC, and LSR across all datasets. This suggests that the doubly stochastic affinity matrices learned by our models are indeed effective as inputs into spectral clustering 
and that our method can also potentially be combined with other subspace clustering methods as a principled normalization step. Beyond clustering accuracy, our models also learn affinity matrices with other desirable properties, 
such as having the lowest SPE among all methods --- indicating that our learned affinities place mass in the correct subspaces.

\vspace{-2mm}

\section{Conclusion}

\vspace{-2mm}

In this work, we propose models based on learning doubly stochastic affinity matrices that unify the self-expressive modeling of subspace clustering with the subsequent spectral clustering. We develop efficient algorithms for computing our models, including a novel active-set method that is highly scalable and allows for clustering large computer vision datasets.
We demonstrate experimentally that our methods achieve substantial improvements in subspace clustering accuracy over state-of-the-art models, without using any of the ad-hoc postprocessing steps that many other methods rely on. Our work further suggests many promising directions for future research by better merging the self-expressive modeling of subspace clustering with spectral clustering (or other graph clustering).

\myparagraph{Acknowledgements.} We thank {Tianjiao Ding}, {Charles R. Johnson}, and {Isay Katsman} for helpful discussions.
This work was partially supported by the Northrop  Grumman Mission Systems Research in Applications for Learning Machines (REALM) initiative, the Johns Hopkins University Research Experience for Undergraduates in Computational Sensing and Medical Robotics (CSMR REU) program, NSF 1704458, and NSF 1934979.

{\small
\bibliographystyle{ieee_fullname}
\bibliography{%
biblio/refs,%
biblio/vidal,%
biblio/sparse,%
biblio/learning,%
biblio/deeplearning,%
biblio/optimization%
}

\clearpage


\begin{appendices}

\section{A-DSSC Details}

\subsection{Derivation of A-DSSC Dual}

Here, we give a derivation for the dual \eqref{eq:dual} of the minimization over $\mbf{A}$ in A-DSSC, as well as the dual with restricted support \eqref{eq:support_dual}. It suffices to derive the restricted support dual, as the full dual follows by setting $\mbf{S}$ as the all ones matrix.
Now, recall that the primal problem is: 
\begin{multline}
	\argmin_{\mbf{A}} \  \langle - \eta_1\abs{\mbf{C}},  \eta_2\mbf{A} \rangle + \frac{\eta_1\eta_2^2}{2}\norm{\mbf{A}}_F^2 \\
	\mrm{s.t.} \ \mbf{A} \in \boldsymbol\Omega_n, \ \mbf{A}\odot\mbf{S}^c = \mbf{0}
\end{multline}
\begin{multline}
	= \argmin_{\mbf{A}} \  \langle - \abs{\mbf{C}},  \mbf{A}\odot\mbf{S} \rangle + \frac{\eta_2}{2}\norm{\mbf{A}\odot\mbf{S}}_F^2 \\
	  \mrm{s.t.} \ \mbf{A}\odot\mbf{S} \in \boldsymbol\Omega_n.
\end{multline}

Introducing Lagrange multipliers $\boldsymbol{\alpha}$ and $\boldsymbol{\beta} \in \RR^n$ for the row and column sum constraints, we have the equivalent problem:
\begin{multline}
	\min_{\mbf{A} \geq \mbf{0}} \max_{\boldsymbol{\alpha}, \boldsymbol{\beta} \in \RR^n} \langle - \abs{\mbf{C}},  \mbf{A}\odot\mbf{S} \rangle + \frac{\eta_2}{2}\norm{\mbf{A}\odot\mbf{S}}_F^2 \\
	+ \langle \boldsymbol{\alpha}, (\mbf{A}\odot\mbf{S})\mbf{1} -\mbf{1} \rangle + \langle \boldsymbol{\beta}, (\mbf{A}\odot\mbf{S})^\top\mbf{1} -\mbf{1} \rangle
\end{multline}
As the primal problem is convex and has a strictly feasible point, strong duality holds by Slater's condition, so this is equivalent to:
\begin{multline}
	\max_{\boldsymbol{\alpha}, \boldsymbol{\beta} \in \RR^n} \min_{\mbf{A} \geq \mbf{0}}  \langle - \abs{\mbf{C}},  \mbf{A}\odot\mbf{S} \rangle + \frac{\eta_2}{2}\norm{\mbf{A}\odot\mbf{S}}_F^2 \\
	+ \langle \boldsymbol{\alpha}, (\mbf{A}\odot\mbf{S})\mbf{1} -\mbf{1} \rangle + \langle \boldsymbol{\beta}, (\mbf{A}\odot\mbf{S})^\top\mbf{1} -\mbf{1} \rangle
\end{multline}
\begin{multline}
	= \max_{\boldsymbol{\alpha}, \boldsymbol{\beta} \in \RR^n} - (\boldsymbol{\alpha} + \boldsymbol{\beta})^\top \mbf{1} +  \min_{\mbf{A} \geq \mbf{0}}  \langle - \abs{\mbf{C}},  \mbf{A}\odot\mbf{S} \rangle + \frac{\eta_2}{2}\norm{\mbf{A}\odot\mbf{S}}_F^2 \\
	+ \langle \boldsymbol{\alpha}\mbf{1}^\top + \mbf{1}\boldsymbol{\beta}^\top, \mbf{A}\odot\mbf{S} \rangle 
\end{multline}
Letting $\mbf{K} = \abs{\mbf{C}} - \boldsymbol{\alpha}\mbf{1}^\top - \mbf{1}\boldsymbol{\beta}^\top$, the inner minimization takes the form:
\begin{align}
	& \min_{\mbf{A} \geq \mbf{0}} \   \langle -\mbf{K} ,  \mbf{A}\odot\mbf{S} \rangle + \frac{\eta_2}{2}\norm{\mbf{A}\odot\mbf{S}}_F^2\\
	= & \eta_2 \min_{\mbf{A} \geq \mbf{0}} \  \left \langle -\frac{1}{\eta_2}\mbf{K} ,  \mbf{A}\odot\mbf{S} \right \rangle + \frac{1}{2}\norm{\mbf{A}\odot\mbf{S}}_F^2\\
	= & -\frac{1}{2\eta_2}\norm{\mbf{K}}_F^2 +  \eta_2 \min_{\mbf{A} \geq \mbf{0}} \ \frac{1}{2}\norm{\frac{1}{\eta_2}\mbf{K} - \mbf{A}\odot\mbf{S}}_F^2 \\
	= & -\frac{1}{2\eta_2}\norm{\mbf{K}}_F^2 +  \frac{1}{2\eta_2}\norm{\mbf{K}_- \odot \mbf{S}^c}_F^2\\
	= & -\frac{1}{2\eta_2}\norm{\mbf{K}_+\odot\mbf{S}}_F^2
\end{align}
Where $[\mbf{K}]_-$ is $\min\{\mbf{K}, \mbf{0}\}$ taken elementwise.
Thus, the final version of the dual is
\begin{equation}
\begin{aligned}
	& \max_{\boldsymbol{\alpha}_\circ, \boldsymbol{\beta}_\circ \in \RR^n} \; -(\boldsymbol{\alpha}_\circ+\boldsymbol{\beta}_\circ)^\top \mbf{1} \\
	& \qquad - \frac{1}{2\eta_2}\norm{\left[|\mbf{C}| - \boldsymbol{\alpha}_\circ \mbf{1}^\top - \mbf{1}\boldsymbol{\beta}_\circ^\top \right]_+ \odot \mbf{S}}_F^2
	\end{aligned}
\end{equation}
and the optimal value of $\mbf{A}$ is as given as
\begin{equation}
	\mbf{A} = \frac{1}{\eta_2} \left[|\mbf{C}|-\boldsymbol{\alpha}_\circ \mbf{1}^\top - \mbf{1} \boldsymbol{\beta}_\circ^\top \right]_+\odot\mbf{S}
\end{equation}

\subsection{Active-Set Method}

Here, we give a proof that our active-set method in Algorithm~\ref{alg:adssc_active} converges to an optimal solution in finitely many steps. We start with a lemma that shows the correctness of our stopping criterion.

\begin{lemma}\label{lem:optimal}
	Let $\mbf{S}$ contain a doubly stochastic zero-nonzero pattern in its support, and let $(\boldsymbol{\alpha}_{\circ}, \boldsymbol{\beta}_{\circ})$ be a solution to the dual of the problem with restricted support \eqref{eq:support_dual} with parameter $\eta_2$. Then $\mbf{A}_\circ = \frac{1}{\eta_2}\left[|\mbf{C}| - \boldsymbol{\alpha}_{\circ}\mbf{1}^\top - \mbf{1}\boldsymbol{\beta}_{\circ}^\top \right]_+$ is optimal for the original problem with unrestricted support if and only if $\mbf{A}_\circ$ is doubly stochastic.
\end{lemma}
\begin{proof}

    It is clear that any optimal solution is doubly stochastic, since these are exactly the feasibility conditions.

	For the other direction, suppose that $\mbf{A}_\circ$ is doubly stochastic. Note that a subgradient of the restricted support dual objective \eqref{eq:support_dual} is given by:
	\begin{align}
		\nabla_{\boldsymbol{\alpha}_\circ} \eqref{eq:support_dual} & = -\mbf{1} + \frac{1}{\eta_2}\left([|\mbf{C}| - \boldsymbol{\alpha}_\circ \mbf{1}^\top - \mbf{1}\boldsymbol{\beta}_\circ^\top]_+ \odot \mbf{S} \right) \mbf{1}\\
		\nabla_{\boldsymbol{\beta}_\circ} \eqref{eq:support_dual} & = -\mbf{1} + \frac{1}{\eta_2}\left([|\mbf{C}| - \boldsymbol{\alpha}_\circ \mbf{1}^\top - \mbf{1}\boldsymbol{\beta}_\circ^\top]_+ \odot \mbf{S} \right)^\top \mbf{1}
	\end{align}
	In the case of unrestricted support (i.e. full support), the subgradients exactly measure how far the row sums and column sums of $\mbf{A}_\circ$ deviate from $\mbf{1}$. Since $\mbf{A}_\circ$ is doubly stochastic, the row and column sums are $1$, so the subgradients are zero. Thus, $\boldsymbol{\alpha}_\circ$ and $\boldsymbol{\beta}_\circ$ are optimal for the unrestricted support dual \eqref{eq:dual}, which means that we may recover the optimal primal variable as $\frac{1}{\eta_2}[|\mbf{C}| - \boldsymbol{\alpha}_\circ \mbf{1}^\top - \mbf{1}\boldsymbol{\beta}^\top]_+$, which is exactly $\mbf{A}_\circ$.
\end{proof}

With this result, we can directly prove convergence and correctness.

\newtheorem*{prop2}{Proposition~\ref{prop:converge}}

\begin{prop2}
    \proptwotext
\end{prop2}
\begin{proof}
	First, note that the initialization of $\mbf{S}$ gives a feasible problem, since the support contains a permutation matrix, which is doubly stochastic. In each iteration, we compute the matrix $\mbf{A}_\circ = \frac{1}{\eta_2}\left[|\mbf{C}| - \boldsymbol{\alpha}_\circ\mbf{1}^\top - \mbf{1}\boldsymbol{\beta}_\circ^\top \right]_+$. If $\mbf{A}_\circ$ is doubly stochastic, then we terminate and $\mbf{A}_\circ$ is optimal for the unrestricted support problem by Lemma~\ref{lem:optimal}. If $\mbf{A}_\circ$ is not doubly stochastic, then $\mbf{A}_\circ \neq \mbf{A}_\circ \odot \mbf{S}$ since $\mbf{A}_\circ \odot \mbf{S}$ is doubly stochastic. Thus, the support is updated to a new support that has not been seen before, since it is strictly larger. There are only finitely many choices of support, so there are only finitely many iterations of this procedure, thus proving finite convergence of our algorithm.
\end{proof}

\subsection{Support Initialization}

The primal \eqref{eq:support_primal} may not be feasible for certain choices of support $\mbf{S}$ with too many zeros, since there is not always a doubly stochastic matrix with a given zero-nonzero pattern. One would like to have an intialization for the support that is both feasible, so that our algorithm converges, and that approximately contains the true support of the optimal $\mbf{A}$, so that the convergence is fast. One reasonable guess for the support of $\mbf{A}$ is the top $k$ entries of each row, since these entries contribute more to the inner product of the objective. Also, for the linear assignment problem, which is the limiting $\eta_2 = 0$ case, the top $k$ entries of each row have been shown to contain the support of the solution $\mbf{A}$ with high probability \cite{aldous2001zeta}. However, a top $k$ graph is not guaranteed to be feasible for reasonably sized $k$, as the following lemma shows:

\begin{lemma}
	Let $n$ and $k$ be integers with $n \geq k \geq 3$.
	\begin{enumerate}
		\item If $k < n/2$, then there exists a graph $G$ on $n$ nodes with minimum degree $k$ such that there is no doubly stochastic matrix with support contained in the support of $G$.
		\item If $k \geq n/2$, then there is a doubly stochastic matrix in the support of any graph on $n$ nodes with minimum degree $k$.
	\end{enumerate}
\end{lemma}
\begin{proof}
    Let $k < n/2$, so that $n-k > k$. Consider the graph with adjacency matrix
	\begin{equation}\label{eq:badgraph}
	    \begin{bmatrix}
			\mbf{0}_{n-k, n-k} & \mbf{1}_{n-k}\mbf{1}_k^\top\\
			\mbf{1}_k\mbf{1}_{n-k}^\top & \mbf{1}_k\mbf{1}_k^\top - \mbf{I}_{k,k}
	    \end{bmatrix}
	\end{equation}
	Each row has either $k$ or $n-1$ ones, so each node has at least degree $k$. However, there can be no permutation matrix with support contained in the support of this graph. This is because the first $n-k$ nodes only have $k$ unique neighbors, and $n-k > k$.

	Now, let $k \geq n/2$, and let $G$ be a graph on $n$ nodes with minimum degree $k \geq n/2$. In each connected component of the graph with $m$ nodes, the minimum degree is at least $n/2 \geq m/2$. Thus, Dirac's theorem on Hamiltonian cycles \cite{dirac1952some} states that each connected component has a Hamiltonian cycle --- a cycle that visits each node in the component. Taking the permutation matrix that is the sum of these cycles gives us a permutation matrix contained in the support of this graph, so we are done since permutation matrices are doubly stochastic.
\end{proof}
One way in which the graph \eqref{eq:badgraph} occurs is if there is a hub and spoke structure in the graph, in which there are $k$ closely connected nodes forming a hub and all other $n-k$ nodes are distant from each other and are connected to the $k$ nodes in the hub. A top-$k$ graph can certainly have this structure if $\mbf{C}$ is a similarity derived from some euclidean distance. In our case, $\mbf{C}$ is a self-expressive affinity, so it is possible that tighter bounds may be derived under some conditions, but we leave this to future work.

Choosing a top-$k$ support initialization for some $k \geq n/2$ defeats the purpose of the active-set method, as computing a solution in this support would still require $\mc O(n^2)$ computations per iteration. In practice, we choose a small $k$ that leads to a tractable problem size, and add some randomly sampled permutation matrices to the support to guarantee feasibility. 

\subsection{Self-Expressive Computation}\label{sec:se_computation}

For a general dual problem $\eqref{eq:support_dual}$ with restricted support $\mbf{S}$, we need only evaluate $\mbf{C}_{ij}$ for $(i,j) \in \mrm{supp}(\mbf{S})$. These $\mbf{C}_{ij}$ can be computed online in each iteration if memory is an issue, or it can be precomputed and stored in $\mc O(|\mbf{S}|)$ memory for faster objective and subgradient evaluations.

In particular, computing a subset of a dense LSR~\cite{Lu12} solution $\mbf{C}$, or computing it online allows us to scale LSR to datasets with many points like MNIST and EMNIST where the entire dense $n$-by-$n$ $\mbf{C}$ cannot be stored. To see how this is done, note that the solution to LSR with parameter $\gamma$ (i.e. a solution to $\eqref{eq:minC}$ with $\eta_1 = \gamma, \eta_3 = 0$) with no zero diagonal constraint is given as
\begin{equation}
	\mbf{C} = (\mbf{X}^\top \mbf{X} + \gamma \mbf{I})^{-1} \mbf{X}^\top \mbf{X}
\end{equation}
The Woodbury matrix identity \cite{horn2012matrix} gives that this is equal to
\begin{align}
	& \left[\frac{1}{\gamma} \mbf{I} - \frac{1}{\gamma^2}\mbf{X}^\top\left(\mbf{I} + \frac{1}{\gamma}\mbf{X}^\top \mbf{X} \right)^{-1} \mbf{X} \right] \mbf{X}^\top \mbf{X}\\
	& = \frac{1}{\gamma}\mbf{X}^\top \mbf{X} - \frac{1}{\gamma^2}\mbf{X}^\top\left(\mbf{I} + \frac{1}{\gamma} \mbf{X}\mbf{X}^\top\right)^{-1}\mbf{X}\mbf{X}^\top \mbf{X}\\
	& = \mbf{X}^\top \underbrace{\left[\frac{1}{\gamma}\mbf{I} - \frac{1}{\gamma^2}\left(\mbf{I} + \frac{1}{\gamma} \mbf{X}\mbf{X}^\top\right)^{-1}\mbf{X}\mbf{X}^\top \right]}_{\mbf{M}} \mbf{X}
\end{align}
so we may compute $\mbf{C}_{ij} = \mbf{X}_i^\top \mbf{M} \mbf{X}_j$. Note that $\mbf{M}$ is $d$-by-$d$, and we do not ever need to form an $n$-by-$n$ matrix in this computation.

\subsection{Active-set Runtime Comparison}\label{appendix:runtime}

The computation of $\mbf{A}$ in A-DSSC \eqref{eq:jminA} is simply application of the proximal operator for projection onto the doubly stochastic matrices $\boldsymbol{\Omega}_n$. To view it as a Frobenius-norm projection of an input matrix onto the doubly stochastic matrices, note that
\begin{align}
& \argmin_{\mbf{A} \in \boldsymbol \Omega_n} \langle -|\mbf{C}|, \mbf{A} \rangle + \frac{\gamma}{2}\norm{\mbf{A}}_F^2 \\ 
= & \argmin_{\mbf{A} \in \boldsymbol \Omega_n} \norm{\frac{1}{\gamma}|\mbf{C}| - \mbf{A}}_F^2
\end{align}
In the past, algorithms for computing this have used the primal formulation.
Zass and Shashua \cite{zass2007doubly} use an alternating projections algorithm, while Rontsis and Goulart \cite{rontsis2020optimal} use an ADMM based algorithm. As explained in Section~\ref{sec:adssc_compute}, we note that the dual \eqref{eq:dual} is a special case of the dual of a regularized optimal transport problem~\cite{blondel2018smooth}, so it can be computed with the same algorithms that these regularized optimal transport problems use. Our Algorithm~\ref{alg:adssc_active} gives a fast way to compute this dual by leveraging problem sparsity.

Here, we compare the performance of these different methods on various input matrices $|\mbf{C}|$. The different input matrices we consider are:
\begin{itemize}
	\item $\mc D_1$, $|\mbf{C}|$ is the A-DSSC affinity for the Yale-B dataset, so $n=2414$. 
	\item $\mc D_2$, $|\mbf{C}|$ is the A-DSSC affinity for scattered COIL, so $n= 7200$.
	\item $\mc D_3$, $|\mbf{C}|$ is $1/2 (|\mbf{G}| + |\mbf{G}^\top|)$, scaled to have maximum element $1$, where $\mbf{G}$ is a standard Gaussian random matrix of size $2000 \times 2000$. We take $\gamma =  .5$.
	\item $\mc D_4$, $|\mbf{C}|$ is an LSR affinity at $\eta_1 = 1$ for a union of ten 5-dimensional subspaces in $\RR^{15}$, as in \cite{you2016divide}. We take $n=4000$ and $\gamma=.01$.
\end{itemize}
For each experiment, we run the algorithm until the row and column sums of the iterates are all within $10^{-4}$ of $1$. Runtimes are given in Table~\ref{tab:runtimes}. Similarly to previous work \cite{rontsis2020optimal}, we find that alternating projections has serious convergence issues; it also fails to converge to numerically sparse affinities. Our active-set method converges orders of magnitudes faster than the other methods. We emphasize that these performance gains are further amplified on larger datasets, where memory considerations can make it so that the other methods cannot even be run. Moreover, all of these experiments were run on CPU, while our method also has a GPU implementation that allows for very efficient computation on large datasets. For instance, we can compute full A-DSSC on scattered MNIST ($n=70{,}000$) and achieve 99\% clustering accuracy on a GeForce RTX 2080 GPU in less than 30 seconds.

\begin{table}
	\caption{Runtime (in seconds) of different doubly stochastic projection algorithms. The data $\mc D_i$ are defined in Section~\ref{appendix:runtime}. ``NC'' indicates that the algorithm did not converge in 5000 iterations.}
	\label{tab:runtimes}
	\centering
	{
	\begin{tabular}{lcccc}
        \toprule
		Method & $\mc D_1$ & $\mc D_2$ & $\mc D_3$ & $\mc D_4$\\ 
		\midrule
		Alt Proj \cite{zass2007doubly} & NC & NC & 118  & 264 \\
		ADMM \cite{rontsis2020optimal} & 54.1  & NC & 26.2 & 111 \\
		Dual \cite{blondel2018smooth} & 11.8 & 568 & 1.69 & 13.7 \\
		Active-set (ours) &  \bf 1.64  & \bf 12.2 & \bf 0.49 & \bf 2.05 \\
\bottomrule
    \end{tabular}
}
\end{table}
	
\section{J-DSSC Details}

\subsection{Linearized ADMM for J-DSSC}

\begin{algorithm}[ht]
	\caption{J-DSSC}
	\label{alg:jdssc}
	\begin{algorithmic}
		\renewcommand{\algorithmicrequire}{\textbf{Input:}}
		\REQUIRE Data matrix $\mbf{X}$, parameters $\eta_1, \eta_2, \eta_3$, step sizes $\rho$ and $\tau$
		\WHILE {Not converged}
		{\small
		\STATE Update $\mbf{C}_p$ and $\mbf{C}_q$ by linearized ADMM steps \eqref{eq:first_lin} to \eqref{eq:last_lin}
		\STATE Update $\mbf{A}, \mbf{Y},$ and $\mbf{Z}$ by minimization steps \eqref{eq:first_min} to \eqref{eq:last_min}
		\STATE Update $\boldsymbol{\lambda}_1, \boldsymbol{\lambda}_2, \boldsymbol{\Lambda}_1, $ and $\boldsymbol{\Lambda}_2$ by dual ascent steps \eqref{eq:first_ascent} to \eqref{eq:last_ascent}
	}
		\ENDWHILE
		\STATE Apply spectral clustering on Laplacian $\mbf{I} - \frac{1}{2}(\mbf{A}+\mbf{A}^\top)$
		\ENSURE Clustering result
	\end{algorithmic}
\end{algorithm}

Here, we give details on the computation of J-DSSC by linearized ADMM.
We reparameterize the problem \eqref{eq:model}
by introducing additional variables $\mbf{Y} = \mbf{A}$ and $\mbf{Z} = \mbf{X}[\mbf{C}_p - \mbf{C}_q]$. This gives us the equivalent problem
\begin{equation}
\begin{aligned}\label{eq:big_problem}
	& \min_{\mbf{C}_p, \mbf{C}_q, \mbf{A}, \mbf{Y}, \mbf{Z}} \  \frac{1}{2}\norm{\mbf{X}-\mbf{Z}}_F^2 \\ 
	& \quad + \frac{\eta_1}{2}\norm{[\mbf{C}_p + \mbf{C}_q] - \eta_2 \mbf{A}}_F^2 + \eta_3 \norm{\mbf{C}_p + \mbf{C}_q}_1\\
	& \mrm{s.t.} \ \mbf{A} \geq \mbf{0}, \  \mbf{C}_p, \mbf{C}_q \in \RR^{n \times n}_{\geq 0, \mrm{diag}=0},\\
	& \quad \mbf{Y}^\top \mbf{1} = \mbf{Y}\mbf{1} = \mbf{1}, \ \mbf{Y} = \mbf{A}, \ \mbf{Z} = \mbf{X}[\mbf{C}_p-\mbf{C}_q]
\end{aligned}
\end{equation}
The augmented Lagrangian then takes the form:
\begin{equation}
\begin{aligned}
 &  \mc L(\mbf{C}_p, \mbf{C}_q, \mbf{A}, \mbf{Y}, \mbf{Z}, \boldsymbol{\lambda}, \boldsymbol{\Lambda}) = \frac{1}{2}\norm{\mbf{X} - \mbf{Z}}_F^2 \\
 & \ + \frac{\eta_1}{2}\norm{[\mbf{C}_p + \mbf{C}_q] - \eta_2 \mbf{A}}_F^2 + \eta_3 \norm{\mbf{C}_p + \mbf{C}_q}_1 \\
 & \ + \langle \boldsymbol{\lambda}_1, \mbf{Y}^\top\mbf{1}-\mbf{1} \rangle + \langle \boldsymbol{\lambda}_2, \mbf{Y}\mbf{1}-\mbf{1} \rangle + \langle \boldsymbol{\Lambda}_1, \mbf{Y}-\mbf{A} \rangle \\
 & \  + \frac{\rho}{2}\norm{\mbf{Y}^\top\mbf{1} - \mbf{1}}_2^2 + \frac{\rho}{2}\norm{\mbf{Y}\mbf{1} - \mbf{1}}_2^2 + \frac{\rho}{2}\norm{\mbf{Y}-\mbf{A}}_F^2\\
 & \  + \langle \boldsymbol{\Lambda}_2, \mbf{Z} - \mbf{X}[\mbf{C}_p - \mbf{C}_q] \rangle + \frac{\rho}{2}\norm{\mbf{Z} - \mbf{X}[\mbf{C}_p - \mbf{C}_q]}_F^2\\
 & \ + \mc I_{\geq 0}(\mbf{A}) + \mc I_{\geq 0, \mrm{diag}=0}(\mbf{C}_p) + \mc I_{\geq 0, \mrm{diag}=0}(\mbf{C}_q) 
\end{aligned}
\end{equation}
where $\mc I_{S}$ is the indicator function for the set $S$,
which takes values of $\mbf{0}$ in $S$ and $\infty$ outside of $S$,
and where $\rho > 0$ is a chosen constant.

In each iteration, we take a linearized ADMM step in $\mbf{C}_p$ and $\mbf{C}_q$,
then alternatingly minimize over $\mbf{A}, \mbf{Y},$ and $\mbf{Z}$,
and lastly take gradient ascent steps on each of the Lagrange multipliers.
For a step size $\tau > 0$, the linearized ADMM step
over $\mbf{C}_p$ takes the form of a gradient descent step on:
\begin{equation}
\begin{aligned}
	h(\mbf{C}_p) = & \langle \boldsymbol{\Lambda}_2, \mbf{Z}-\mbf{X}[\mbf{C}_p - \mbf{C}_q] \rangle \\
	& \qquad + \frac{\rho}{2}\norm{\mbf{Z} - \mbf{X}[\mbf{C}_p - \mbf{C}_q]}_F^2,
\end{aligned}
\end{equation}
followed by application of a proximal operator. Letting $\mbf{C}_p'$ be the
intermediate value of $\mbf{C}_p$ after the gradient descent step, $\mbf{C}_p$ is 
updated as the solution to
\begin{multline}
	\min_{\mbf{C}_p} \frac{\eta_1}{2}\norm{[\mbf{C}_p + \mbf{C}_q] - \eta_2 \mbf{A}}_F^2 + \eta_3 \norm{\mbf{C}_p + \mbf{C}_q}_1\\
	+ \frac{1}{2\tau}\norm{\mbf{C}_p ' - \mbf{C}_p}_F^2 + \mc I_{\geq 0, \mrm{diag}=0}(\mbf{C}_p)
\end{multline}
For a matrix $\mbf{E}$, let $[\mbf{E}]_+$ be the half-wave rectification of
$\mbf{E}$,
and let $[\mbf{E}]_{+, d=0}$ be the matrix $\mbf{E}$ with all negative entries
and the diagonal set to zero. Then the linearized ADMM updates are given as follows:
{\small
\begin{align}
	\mbf{C}_p' & \gets \mbf{C}_p - \tau \left(-\mbf{X}^\top \boldsymbol{\Lambda}_2 + \rho \mbf{X}^\top(\mbf{X}[\mbf{C}_p - \mbf{C}_q] - \mbf{Z}) \right) \label{eq:first_lin}\\
	\mbf{C}_p & \gets \left[ \frac{1}{\eta_1 + \frac{1}{\tau}}\left(\frac{1}{\tau} \mbf{C}_p' - \eta_1 \mbf{C}_q + \eta_1 \eta_2 \mbf{A} - \eta_3 \mbf{11}^\top \right) \right]_{+, \mrm{d}=0}\\
	\mbf{C}_q' & \gets \mbf{C}_q - \tau(\mbf{X}^\top \boldsymbol{\Lambda}_2 - \rho \mbf{X}^\top(\mbf{X}[\mbf{C}_p - \mbf{C}_q] - \mbf{Z})) \\
	\mbf{C}_q & \gets \left[\frac{1}{\eta_1 + \frac{1}{\tau}}\left(\frac{1}{\tau}\mbf{C}_q' - \eta_1 \mbf{C}_p + \eta_1 \eta_2 \mbf{A} - \eta_3 \mbf{11}^\top\right) \right]_{+, \mrm{d}=0} \label{eq:last_lin}
\end{align}
}
The sequential minimizations over $\mbf{A}$ then $\mbf{Y}$ then $\mbf{Z}$,
holding the other variables
fixed and using the most recent value of each variable, take the form:
\begin{align}
	\mbf{A} & \gets \left[\frac{1}{\eta_1 \eta_2^2 + \rho} (\eta_1 \eta_2 [\mbf{C}_p + \mbf{C}_q] + \boldsymbol{\Lambda}_1 + \rho \mbf{Y}) \right]_+ \label{eq:first_min} \\
	\mbf{Y} & \gets \frac{1}{\rho} \left[ \mbf{V} - \frac{1}{n+1}\mbf{PV 11}^\top - \mbf{11^\top V} \frac{1}{n+1}\mbf{P} \right]\\
			& \text{where } \mbf{V}  = \rho \mbf{A} + 2\rho\mbf{11}^\top - \mbf{1}\boldsymbol{\lambda}_1^\top - \boldsymbol{\lambda}_2\mbf{1}^\top - \boldsymbol{\Lambda}_1\\
			& \hspace{28pt} \mbf{P} = \mbf{I} - \frac{1}{2n+1} \mbf{11}^\top\\
	 \mbf{Z} & \gets \frac{1}{1+\rho}\left(\mbf{X} - \boldsymbol{\Lambda}_2  + \rho \mbf{X}[\mbf{C}_p - \mbf{C}_q]\right) \label{eq:last_min}
 \end{align}
The dual ascent steps on the Lagrange multipliers take the form:
\begin{align}
	\boldsymbol{\lambda}_1 & \gets \boldsymbol{\lambda}_1 + \rho (\mbf{Y}^\top\mbf{1} - \mbf{1}) \label{eq:first_ascent}\\
	\boldsymbol{\lambda}_2 & \gets \boldsymbol{\lambda}_2 + \rho (\mbf{Y}\mbf{1} - \mbf{1})\\
	\boldsymbol{\Lambda}_1 & \gets \boldsymbol{\Lambda}_1 + \rho (\mbf{Y}-\mbf{A})\\
	 \boldsymbol{\Lambda}_2 & \gets \boldsymbol{\Lambda}_2 + \rho (\mbf{Z} - \mbf{X}[\mbf{C}_p - \mbf{C}_q]) \label{eq:last_ascent}
\end{align}
Thus, the full J-DSSC algorithm is given in Algorithm \ref{alg:jdssc}.

To derive the minimization steps for a variable $\mbf{E}$ that is either $\mbf{C}_p, \mbf{C}_q, \mbf{A},$ or $\mbf{Z}$, we use the method of completing the square to change the corresponding minimization into one of the form $\min_{\mbf{E} \in S} \norm{\mbf{E} - \mbf{E}'}_F^2$ for some feasible set $S$ and some matrix $\mbf{E}'$. For instance, in the case of $\mbf{C}_p$, the set $S$ is the set of $n \times n$ nonnegative matrices with zero diagonal, and the solution is $[\mbf{E}']_{+, d=0}$.

To minimize over $\mbf{Y}$, we use different approaches often used
in dealing with row or column sum constraints \cite{zass2007doubly, wang2016structured}.
Define ${\mbf{V} = \rho\mbf{A} + 2\rho\mbf{1}\mbf{1}^\top - \mbf{1}\boldsymbol{\lambda}_1^\top - \boldsymbol{\lambda}_2\mbf{1}^\top - \boldsymbol{\Lambda}_1}$. Setting the gradient of the augmented Lagrangian with respect to $\mbf{Y}$ equal to zero, we can show that
\begin{equation}\label{eq:solveY}
	\rho\mbf{Y} + \rho\mbf{Y}\mbf{1}\mbf{1}^\top + \rho\mbf{1}\mbf{1}^\top\mbf{Y} = \mbf{V}
\end{equation}
Multiplying both sides by $\mbf{1}$ on the right and noting that $\mbf{1}^T\mbf{1} =n$, we have
\begin{align}
	& \rho(\mbf{I} + n\mbf{I} + \mbf{1}\mbf{1}^\top)\mbf{Y}\mbf{1}  = \mbf{V}\mbf{1}\\
	& \rho\mbf{Y}\mbf{1} = (\mbf{I} + n\mbf{I} + \mbf{1}\mbf{1}^\top)^{-1}\mbf{V}\mbf{1} \label{eq:rightOne}
\end{align}
Likewise, multiplying by $\mbf{1}^\top$ on the left gives that
\begin{equation}\label{eq:leftOne}
	\rho\mbf{1}^\top\mbf{Y} = \mbf{1}^\top\mbf{V} (\mbf{I} + n\mbf{I} + \mbf{1}\mbf{1}^\top)^{-1}
\end{equation}
We use the Woodbury matrix identity to compute the inverse
\begin{equation}
	(\mbf{I} + n\mbf{I} + \mbf{1}\mbf{1}^\top)^{-1} = \frac{1}{n+1}\left(\mbf{I} - \frac{1}{2n+1}\mbf{1}\mbf{1}^\top \right).
\end{equation}
We define $\mbf{P} = \mbf{I} - \frac{1}{2n+1}\mbf{1}\mbf{1}^\top$, so the inverse is $\frac{1}{n+1}\mbf{P}$. Thus, using both \eqref{eq:rightOne} and \eqref{eq:leftOne} in the left hand side of \eqref{eq:solveY} gives that
\begin{equation}
	\rho\mbf{Y} + \frac{1}{n+1}\mbf{P}\mbf{V}\mbf{1}\mbf{1}^\top + \mbf{1}\mbf{1}^\top \mbf{V} \frac{1}{n+1}\mbf{P} = \mbf{V}\\
\end{equation}
and thus the minimizing $\mbf{Y}$ is given as
\begin{equation}
	\mbf{Y} = \frac{1}{\rho}\left[\mbf{V} - \frac{1}{n+1}\mbf{P}\mbf{V}\mbf{1}\mbf{1}^\top - \mbf{1}\mbf{1}^\top \mbf{V} \frac{1}{n+1}\mbf{P} \right].
\end{equation}

\section{Implementation Details}

\noindent We implement J-DSSC in the Julia programming language. For A-DSSC, we have separate implementations in Julia and Python, where the Python implementation uses PyTorch for GPU computations.
For the J-DSSC algorithm,
we take the step sizes to be $\rho=.5$ and $\tau=.0001$ for our experiments.

\subsection{Complexity Analysis}

Let $\mbf{X} \in \RR^{d \times n}$, so there are $n$ data points in $\RR^d$. For J-DSSC, each update on $\mbf{C}_p$ and $\mbf{C}_q$ takes $\mc O(dn^2)$ operations due to matrix multiplications. Each update on $\mbf{A}$ and $\mbf{Y}$ takes $\mc O(n^2)$ operations, while each update on $\mbf{Z}$ takes $\mc O(d n^2)$.
The dual ascent steps also take $\mc O(dn^2)$ operations in total.
Thus, each iteration of the linearized ADMM part takes $\mc O(dn^2)$. We
note in particular that J-DSSC does not require the solution of any linear systems
or any singular value decompositions, while many other subspace clustering methods do.

For A-DSSC, solving for $\mbf{C}$ can be done by some efficient EnSC, SSC, or LSR solver. As discussed in Section~\ref{sec:se_computation}, we need only compute $\mbf{C}_{i,j}$ for $(i,j) \in \mrm{supp}(\mbf{S})$ of some support $\mbf{S}$, and this can be done online or as a precomputation step.
The quadratically regularized optimal transport step takes $\mc O(|\mbf{S}|)$ operations per 
objective function evaluation and gradient evaluation.

The spectral clustering step requires a partial eigendecomposition of $\mbf{I}- \frac{1}{2}(\mbf{A}+\mbf{A}^\top)$, although we note that in practice the $\mbf{A}$ that is learned tends to be sparse for the parameter settings that we choose, so we can leverage sparse eigendecomposition algorithms to compute the $k$ eigenvectors corresponding to the smallest eigenvalues. Recall that all subspace clustering methods that use spectral clustering (which is the majority of methods) require computing an eigendecomposition of equivalent size, though not all subspace clustering affinities can leverage sparse eigendecomposition methods.

\subsection{Model Selection}

Our method J-DSSC has three hyperparameters $\eta_1, \eta_2, \eta_3$ to set. If we fix $\mbf{A}$, the problem is simply EnSC \cite{you2016oracle}. This means that we may choose the parameters $\eta_1$ and $\eta_3$ in J-DSSC as is done in EnSC.
As, noted in Section~\ref{sec:jdssc}, $\eta_2$ controls the sparsity of the final $\mbf{A}$, with smaller $\eta_2$ tending to give sparser $\mbf{A}$. Thus, we choose $\eta_2$ such that $\mbf{A}$ is within a desired sparsity range. 
To avoid solutions that are too sparse, we can consider the number of connected components of the graph associated to the solution $\mbf{A}$; if we want to obtain $k$ clusters, we can adjust $\eta_2$ so that the number of connected components is at most $k$.

The parameters for A-DSSC can be selected in an analogous manner to those of J-DSSC. Another interesting consideration is that A-DSSC does very well when $\eta_3 = 0$, meaning that there is no $l_1$ regularization on $\mbf{C}$ and thus $\mbf{C}$ is the solution to an LSR problem \cite{Lu12}. In this case, $\mbf{C}$ can be computed efficiently with one linear system solve, and we need only set the two parameters $\eta_1$ and $\eta_2$.

\begin{table*}[h]
	\centering
\caption{Clustering accuracy comparison with neural network models. `---' indicates that we do not have results for the method. We take results directly from the corresponding papers (besides DSC-Net on UMIST, which we run). Note that this is not under our main paper evaluation framework, as we cannot control postprocessing of neural models \cite{haeffele2021critique}.} 
	{\footnotesize
	\begin{tabular}{cccccccccc}
		\toprule
				& \multicolumn{5}{c}{Neural} &  \multicolumn{2}{c}{Doubly Stochastic} \\
				\cmidrule(lr){2-6}  \cmidrule(lr){7-8}
		Dataset &  DSC-Net \cite{ji2017deep}  & DASC \cite{zhou2018deep} & $k$SCN\cite{zhang2018scalable} & S\supscript{2}ConvSCN \cite{zhang2019self} & MLRDSC-DA \cite{abavisani2020deep} & J-DSSC (Ours) &  A-DSSC (Ours) \\
\midrule
		Yale-B	& .973 & .986 & --- & .985 & \bf .992 & .924 & .917  \\
		COIL-100 & .690 & --- & --- & .733 & .793 & .961 & \bf .984  \\
		ORL	& .860 & .883 & --- &  .895 & \bf .897 &  .785 &  .790  \\
		UMIST & .708 & .769 & --- & --- & --- & .873 & \bf .888 \\
		MNIST & --- & --- & .871 & --- & --- & --- & \bf .990 \\
		\bottomrule
	\end{tabular}
}
\vspace{-1pt}
	\label{tab:neuralresults}
\end{table*}

\section{Proof of Proposition 1}

\begin{proof}
	Suppose that there is an index $(i,j)$ such that 
	\begin{equation}
	a := (\mbf{C}_p^*)_{ij} > 0, \quad b := (\mbf{C}_q^*)_{ij} > 0.
\end{equation}
	Define matrices $\mbf{C}_p'$ and $\mbf{C}_q'$ such that $\mbf{C}_p'$ matches $\mbf{C}_p^*$ at every index but $(i,j)$ and similarly for $\mbf{C}_q'$. At index $(i,j)$, we define 
	\begin{equation}
	(\mbf{C}_p')_{ij} = \max(a-b, 0), \quad (\mbf{C}_q')_{ij} = \max(b-a, 0)
\end{equation}
We consider the value of the expanded J-DSSC objective \eqref{eq:expand}
at $(\mbf{C}_p', \mbf{C}_q', \mbf{A}^*)$ compared to the value at
$(\mbf{C}_p^*, \mbf{C}_q^*, \mbf{A}^*)$.
Note that $\mbf{C}_p' - \mbf{C}_q' = \mbf{C}_p^* - \mbf{C}_q^*$. Thus, there are only three summands
in the objective that are different between the two groups of variables. 
Suppose $a \geq b$. The differences are computed as:
\begin{align}
	& \frac{\eta_1}{2}\norm{\mbf{C}_p' + \mbf{C}_q'}_F^2 - \frac{\eta_1}{2}\norm{\mbf{C}_p^* + \mbf{C}_q^*}_F^2 \nonumber\\
	& \qquad \qquad = \frac{\eta_1}{2}\left(\abs{a-b}^2 - (a+b)^2\right)\\
	& \qquad \qquad = -2\eta_1 ab \nonumber\\
	& -\eta_1\langle  [\mbf{C}_p' + \mbf{C}_q'], \eta_2 \mbf{A}^* \rangle + \eta_1 \langle [\mbf{C}_p^* + \mbf{C}_q^*], \eta_2 \mbf{A}^* \rangle\nonumber\\
	& \qquad \qquad = \eta_1 \eta_2 (\mbf{A}^*)_{ij}\left( -\abs{a-b} + (a+b) \right)\\
	& \qquad \qquad = 2\eta_1 \eta_2 (\mbf{A}^*)_{ij}b \nonumber\\
	& \eta_3 \norm{\mbf{C}_p' + \mbf{C}_q'}_1 - \eta_3 \norm{\mbf{C}_p^* + \mbf{C}_q^*}_1 \nonumber\\
	& \qquad \qquad = \eta_3 \left(\abs{a-b} - (a+b) \right)\\
	& \qquad \qquad = -2\eta_3 b \nonumber
\end{align}
Now, we bound the difference between the objective functions for the two variables:
\begin{align*}
	0 & < -2\eta_1 ab + 2\eta_1 \eta_2 (\mbf{A}^*)_{ij} b - 2\eta_3 b\\
	  & \leq -2\eta_1 ab + 2 \eta_1 \eta_2 b - 2 \eta_3 b
\end{align*}
where the first bound is due to $(\mbf{C}_p^*, \mbf{C}_q^*, \mbf{A}^*)$ being the optimal solution, and the upper bound is due to $(\mbf{A}^*)_{ij} \leq 1$. From these bounds, we have that
\begin{equation}
	b \leq a < \frac{\eta_1 \eta_2 - \eta_3}{\eta_1}.
\end{equation}
Where we recall that we assumed $b \leq a$.
A similar upper bound holds with $b$ taking the place of $a$
in the case where $b > a$. Now, note that $a > 0$ means that
$\eta_1 \eta_2 > \eta_3$, thus proving 1) through the contrapositive.
This bound clearly gives the one in \eqref{eq:overlap_support}, so we have also proven~2).
\end{proof}

\section{Further Experiments}

\subsection{Comparison with Neural Networks}

In Table~\ref{tab:neuralresults}, we list the clustering accuracy of subspace clustering network models against our DSSC models. These results are taken directly from the corresponding papers, which means that they do not follow our common evaluation framework as used in Section~\ref{sec:experiments} of the main paper. Thus, the neural network methods may still use their ad-hoc postprocessing methods, which have been shown to significantly affect empirical performance \cite{haeffele2021critique}. Nonetheless, DSSC still significantly outperforms neural methods on COIL-100, UMIST, and MNIST. In fact, self-expressive neural models generally cannot compute clusterings for datasets on the scale of MNIST or EMNIST, as they require a dense $n$-by-$n$ matrix of parameters for the self-expressive layer.

\subsection{Behavior of DSSC Models}

\begin{figure*}[ht]
	\centering
	\subfloat[$\eta_2= .1$]{\includegraphics[width=0.37\columnwidth]{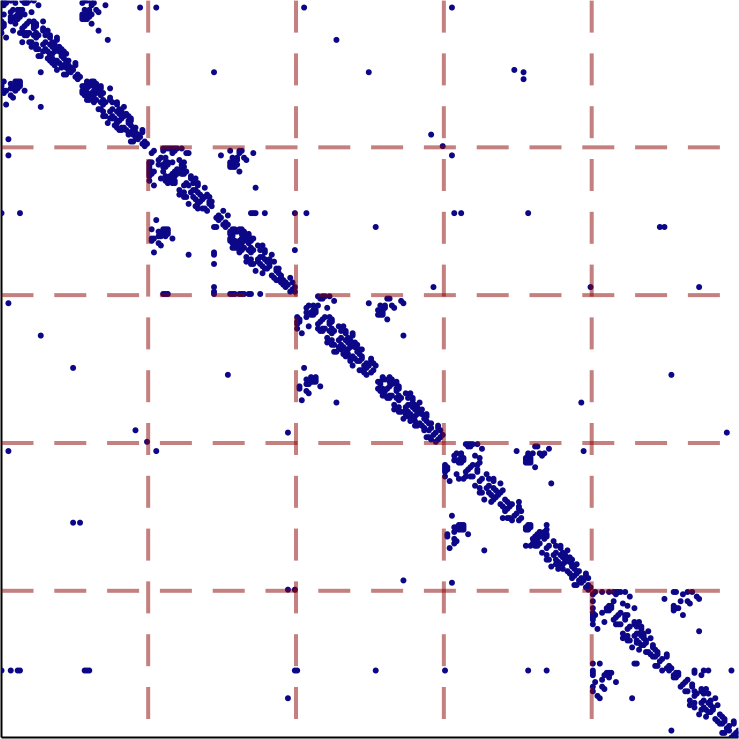}} \hspace{15pt}
	\subfloat[$\eta_2= .2$]{\includegraphics[width=0.37\columnwidth]{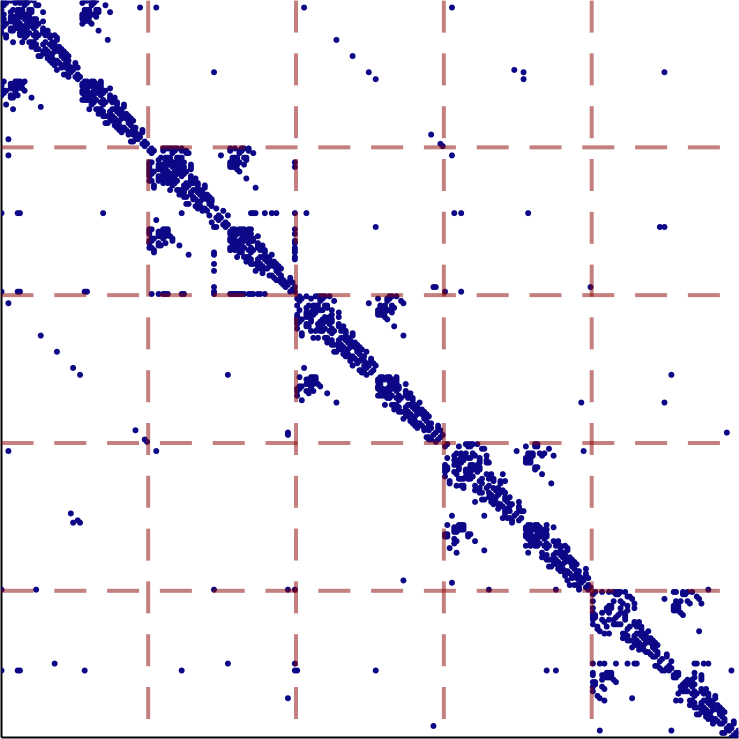}} \hspace{15pt}
	\subfloat[$\eta_2= .5$]{\includegraphics[width=0.37\columnwidth]{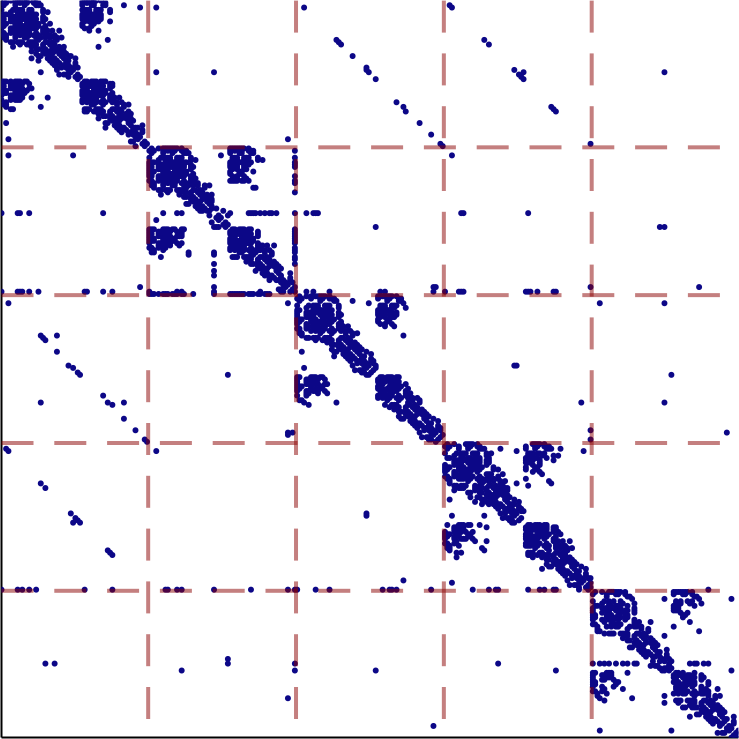}} \hspace{15pt}
	\subfloat[$\eta_2= 1.0$]{\includegraphics[width=0.37\columnwidth]{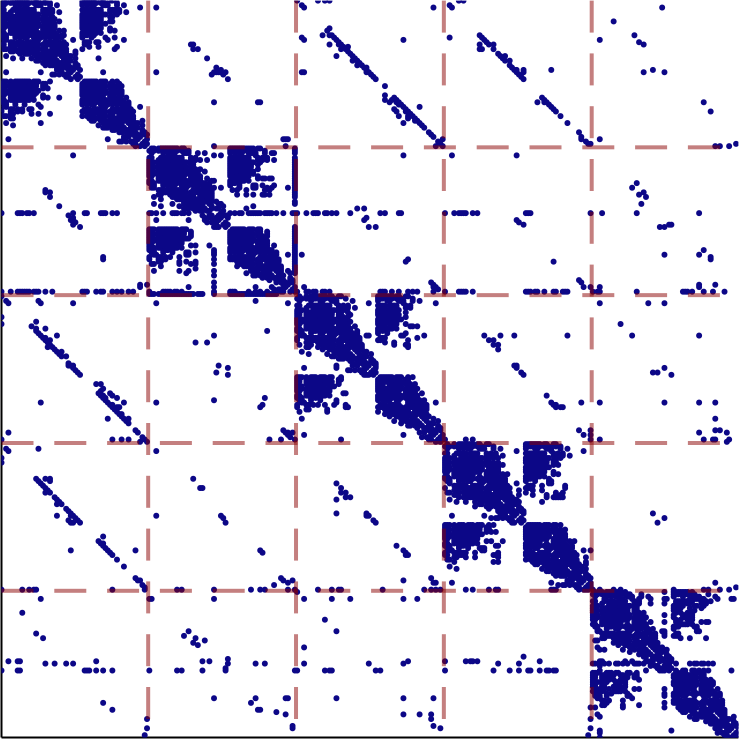}}

	\caption{Affinity matrices $\mbf{A}$ learned by J-DSSC on the first five classes of Yale-B. We fix $\eta_1 = .25$ and $\eta_3 = 0$ (so there is no $l_1$ regularization on $\mbf{C}$), while varying $\eta_2$. Red dashed lines mark the boundaries between classes. Dark blue points mark nonzero entries in $\mbf{A}$. The learned affinity $\mbf{A}$ is sparser as $\eta_2$ decreases, while it is more connected as $\eta_2$ increases.}
	\label{fig:sparsity}
\end{figure*}

\begin{figure}[t]
	\centering
	\includegraphics[width=.48\columnwidth]{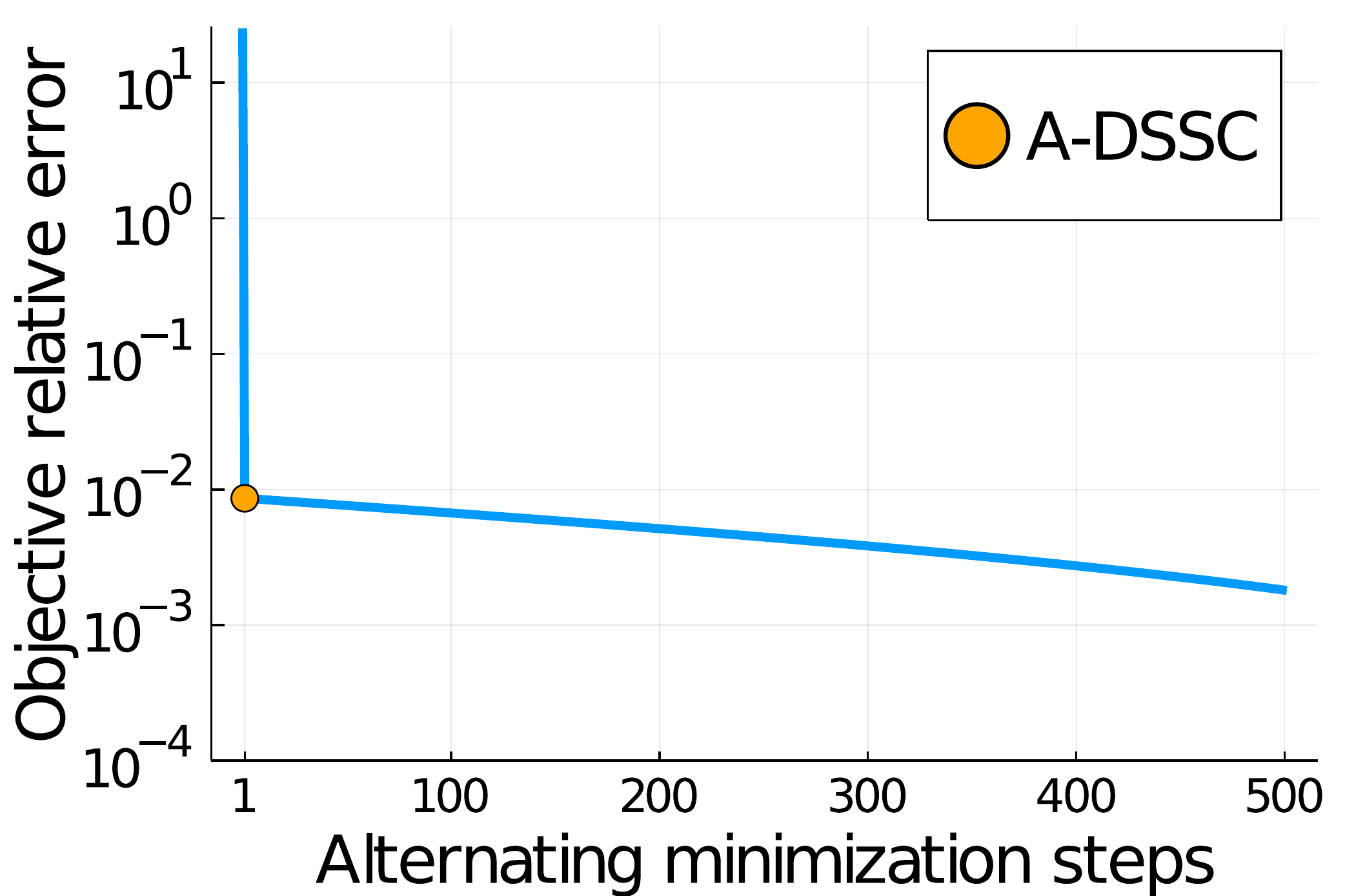} \hfill
	\includegraphics[width=.48\columnwidth]{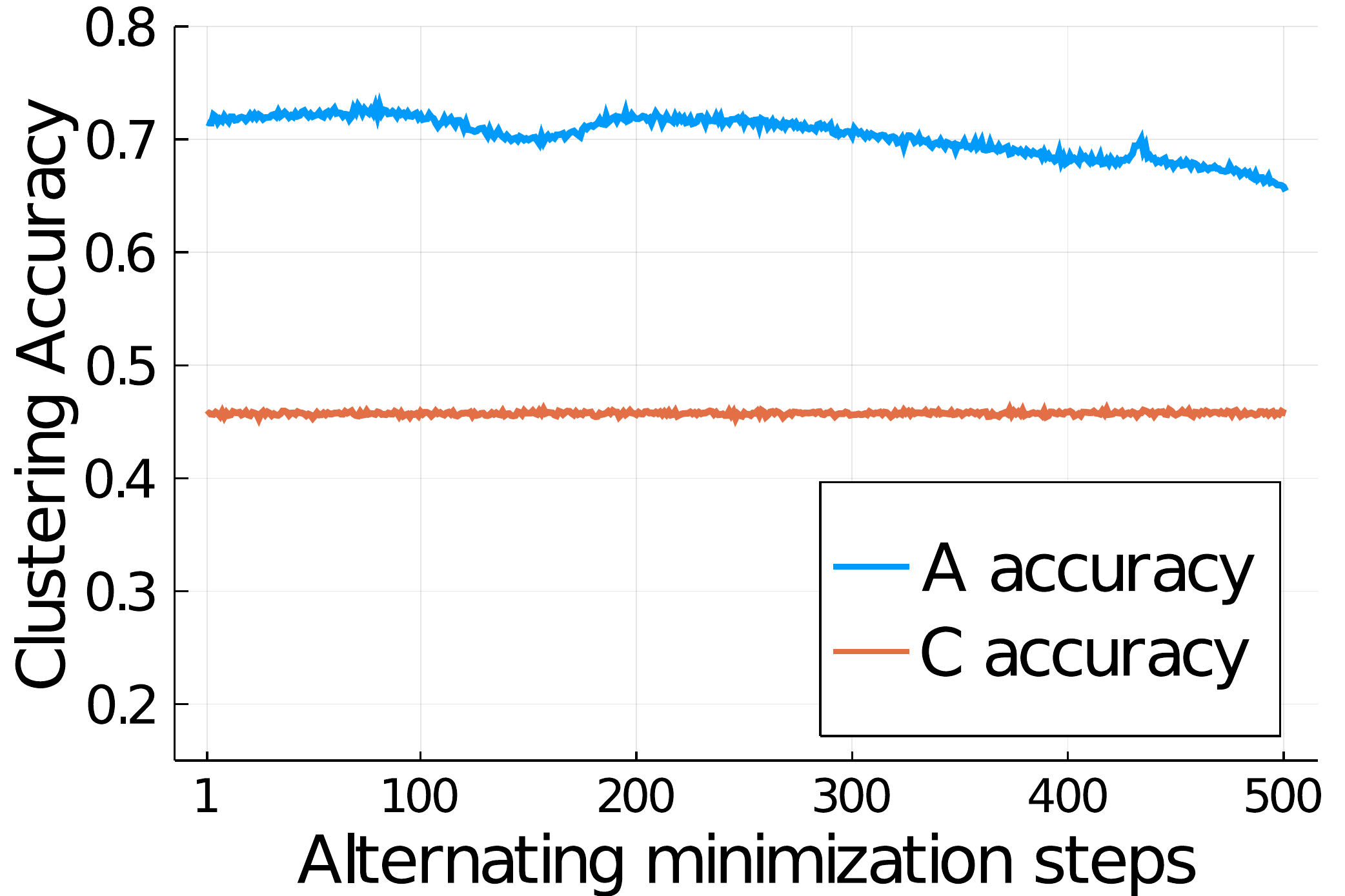}
	\caption{Objective value and clustering accuracy for different number of
	alternating minimization steps over $\mbf{C}$ in \eqref{eq:minC} then $\mbf{A}$ in \eqref{eq:jminA} on the UMIST dataset. Our A-DSSC model is equivalent to 1 alternating minimization step.
	(Left) Shows relative error in the objective value compared to the true minimum.
	(Right) Shows clustering accuracy of using either $\mbf{A}$ or $\mbf{C}$ in spectral clustering.
	}
	\label{fig:objective}
\end{figure}

Here, we explore the quality of our approximation A-DSSC by comparing the natural extension of this approximation to multiple steps of alternating minimization over $\mbf{C}$ in \eqref{eq:minC} and $\mbf{A}$ in \eqref{eq:jminA}. 
In particular, we note that additional alternating minimization steps do not bring a significant benefit over just a single alternating minimization iteration. Figure \ref{fig:objective} displays the results of our A-DSSC model on the UMIST face dataset \cite{graham1998characterising} using the experimental setup described in Section \ref{sec:setup}.  We see that a single alternating minimization step over $\mbf{C}$ then $\mbf{A}$, which is equivalent to our A-DSSC model, already achieves most of the decrease in the objective \eqref{eq:model} --- giving a relative error of less than 1\% compared to the global minimum obtained by the full J-DSSC solution.
Moreover, the clustering accuracy is robust to additional alternating minimization steps beyond a single A-DSSC step ---
additional steps do not increase clustering accuracy from spectral clustering on $\mbf{A}$. Also, spectral clustering on $\mbf{A}$ achieves much higher accuracy than spectral clustering on $\mbf{C}$, thus showing the utility of the doubly stochastic model.

We also investigate the sparsity and connectivity of the affinity $\mbf{A}$ learned by our models.
In Figure \ref{fig:sparsity}, we show sparsity plots
of the learned matrices $\mbf{A}$ for J-DSSC models with varying $\eta_2$
on the first five classes on 
the Extended Yale-B dataset \cite{georghiades2001few}.
The points are sorted by class, so intra-subspace
connections are on the block diagonal, and inter-subspace connections are on the off-diagonal blocks.
Here, it can be seen that the sparsity level of the recovered $\A$ affinity is controlled by choice of the $\eta_2$ parameter even when there is no sparse regularization on $\C$ (i.e., $\eta_3=0$), with the number of nonzeros increasing as $\eta_2$ increases.

\section{Data Details}

\begin{figure}[ht]
	\centering
	\includegraphics[width=.8\columnwidth]{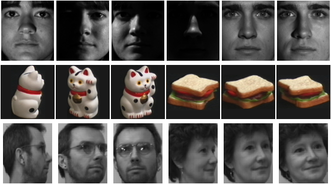}
	\caption{Sample images from Extended Yale-B (top row), COIL-100 (middle row), and UMIST (bottom row).} 
	\label{fig:images}
\end{figure}

\begin{table}[ht]
	\centering
	\caption{Statistics and properties of datasets.}
	\label{tab:dataStats}
	{\small
	\begin{tabular}{lccc}
		\toprule
		Dataset & \# Images ($n$) & Dimension & \# Classes \\
		\midrule
		Yale-B & 2,414 & $48 \times 42$ & 38 \\
		COIL-40 & 2,880 & $32 \times 32$ & 40 \\
		COIL-100 & 7,200 & $32 \times 32$ & 100 \\
		UMIST & 575 & $32 \times 32$ & 20 \\
		ORL & 400 & $32 \times 32$ & 40  \\
		MNIST & 70,000 & $28 \times 28$ & 10 \\
		EMNIST & 145,600 & $28 \times 28$ & 26 \\
		\bottomrule
	\end{tabular}
}
\end{table}

For our experiments, recall that we run subspace clustering on the Extended Yale-B face dataset \cite{georghiades2001few}, Columbia Object Image Library (COIL-40 and COIL-100) \cite{nenecolumbia}, UMIST face dataset \cite{graham1998characterising}, ORL face dataset \cite{samaria1994parameterisation}, MNIST digits dataset \cite{lecun1998gradient}, and EMNIST-Letters dataset \cite{cohen2017emnist}. 
Sample images from some of these datasets are shown in Figure~\ref{fig:images}.
Basic statistics and properties of these datasets are given in Table~\ref{tab:dataStats}.

The scattered data is computed as in \cite{you2016oracle, you2016scalable},
by using a scattering convolution network \cite{bruna2013invariant} 
to compute a feature vector of 3{,}472 dimensions for UMIST, MNIST, and EMNIST, and $3\cdot 3{,}472 = $ 10{,416}
dimensions for COIL (since there are three color channels), then projecting to
500 dimensions by PCA. 
For both pixel data and scattered data, we
normalize each data point to have unit $l_2$ norm for all experiments,
noting that scaling a data point does not change its subspace membership.

For Yale-B, we remove the 18 images labeled as corrupted in the dataset, leaving
2414 images for our experiments. Adding the corrupted images back in does not
appear to qualitatively change our results. We resize Yale-B to size $48 \times 42$.
The COIL, UMIST, and ORL images are resized to size $32 \times 32$.

\section{Parameter Settings}

\begin{table}[ht]
	\centering
	\caption{Parameter settings for J-DSSC and A-DSSC.}
	\label{tab:paramDSSC}
	{\footnotesize
		\begin{tabular}{lcccccc}
		\toprule
		& \multicolumn{3}{c}{J-DSSC} & \multicolumn{3}{c}{A-DSSC}\\
				\cmidrule(lr){2-4} \cmidrule(lr){5-7} 
		Dataset & $\eta_1$ & $\eta_2$ & $\eta_3$ & $\eta_1$ & $\eta_2$ & $\eta_3$\\
		\midrule
		Yale-B & .25 & .2 & 0 & .5 & .1 & 0 \\
		COIL-40 & 25 & .01 & .1 & 25 & .001 & 0  \\
		COIL-40 (Scattered) & .25 & .2 & 0 & 50 & .001 & 0 \\
		COIL-100 & 25 & .01 & .1 & 50 & .0005 & 0 \\
		COIL-100 (Scattered) & .25 & .1 & 0 & .1 & .025 & 0 \\
		UMIST & 1 & .05 & 0 & .5 & .05 & 0 \\
		UMIST (Scattered) & .01 & .2 & 0 & .5 & .01 & 0 \\
		ORL & 1 & .1 & .1 & 1 & .05 & 0\\
		MNIST (Scattered) & N/A & N/A & N/A & 10 & .001 & 0   \\
		EMNIST (Scattered) & N/A & N/A & N/A & 50 & .001 & 0 \\
		\bottomrule
    \end{tabular}
	}
\end{table}

\begin{table}[ht]
	\caption{Parameters searched over for different methods.}
	\label{tab:paramShallow}
	{\small
		\begin{tabular}{lp{5.5cm}}
		\toprule
			Method & Parameters and Variant \\
		\midrule
			SSC \cite{elhamifar2009sparse} & $\gamma \in \{1, 5, 10, 25, 50, 100 \}$, noisy data variant  \\
			\midrule
			EnSC \cite{you2016oracle} & $\gamma \in \{.1, 1, 5, 10, 50, 100 \}$, $\lambda \in \{.9, .95\}$ \\
			\midrule
			LSR \cite{Lu12} & ${\lambda \in \{.01, .1, .5, 1, 10, 50, 100\}}$, $\mrm{diag}(\mbf{C}) = \mbf{0}$ variant \\
			\midrule
			LRSC \cite{vidal2014low} & $\tau \in \{.1, 1, 10, 50, 100, 500, 1000\}, \alpha=\tau/2$, noisy data and relaxed constraint variant\\
			\midrule
			TSC \cite{heckel2015robust} &  $ q \in \{2, 3, \ldots, 14, 15 \}$ \\
			\midrule
			SSC-OMP \cite{you2016scalable} & $k_{\mrm{max}} \in \{2, 3, \ldots, 14, 15 \}$\\
			\midrule
			J-DSSC (Ours) & ${\eta_1 \in \{.01, .25, 1, 25\}}$, ${\eta_2 \in \{.01, .05, .1, .2\} }$, ${\eta_3 \in \{0, .1\}}$\\
			\midrule
			A-DSSC (Ours) & ${\eta_1 \in \{.1, 1, 10, 25, 50\}}$, ${\eta_2 \in \{.0005, .001, .01, .025, .05, .1\}}$, $\eta_3 = 0$ \\
		\bottomrule
    \end{tabular}
	}
\end{table}

Here, we give parameter settings for the experiments in the paper.
The chosen parameter settings for our DSSC models are given in 
Table~\ref{tab:paramDSSC}. We note that the experiments on UMIST
to explore the objective value over iterations of alternating minimization
in Figure~\ref{fig:objective} use the same parameter settings
as in this table.

The parameters searched over for the non-neural methods
are given in Table~\ref{tab:paramShallow}.

For DSC-Net, we use the hyperparameters
and pre-trained autoencoders as in their original paper \cite{ji2017deep}
for Yale-B, COIL-100, and ORL.
For DSC-Net on UMIST, we take a similar architecture to that of \cite{zhou2018deep},
with encoder kernel sizes of $5\times 5, 3 \times 3, 3 \times 3$, as well as 15, 10, 5 channels
per each layer. The regularization parameters are set as $\lambda_1 = 1$ and $\lambda_2 = .2$.
The DSC-Net postprocessing on UMIST is taken to be similar to that used for ORL.
For COIL-40, we use one encoder layer of kernel size $3 \times 3$, 20 channels, and
regularization parameters $\lambda_1 = 1, \lambda_2 = 100$; we also use the same postprocessing
that is used for COIL-100.

\end{appendices}

\end{document}